\documentclass[journal,onecolumn]{IEEEtran}
\usepackage{graphicx}% Include figure files
\usepackage{dcolumn}% Align table columns on decimal point
\usepackage{bm}% bold math
\usepackage{comment}
\usepackage{fancyhdr}
\usepackage{latexsym}
\usepackage{graphicx,color}
\usepackage{epstopdf}  %%after  \usepackage{graphicx,color}
\usepackage[pdfstartview=FitH]{hyperref}
\usepackage{xargs}
\usepackage{shadow,epsf,amsthm,amssymb,amsmath,caption,mathtools}
\usepackage{cleveref}

\usepackage{threeparttable}
\usepackage{multirow}
\usepackage{float}
\usepackage{makecell}
\usepackage{booktabs, tabularx, colortbl}
\usepackage{array}
\usepackage{subfigure}
\newcolumntype{C}{>{$}c<{$}}
\AtBeginDocument{
	\heavyrulewidth=.08em
	\lightrulewidth=.05em
	\cmidrulewidth=.03em
	\belowrulesep=.65ex
	\belowbottomsep=0pt
	\aboverulesep=.4ex
	\abovetopsep=0pt
	\cmidrulesep=\doublerulesep
	\cmidrulekern=.5em
	\defaultaddspace=.5em
}

\usepackage{enumerate}

\usepackage{setspace}                           %for doublespacing

\usepackage{lipsum}% just to automatically generate text

\usepackage{tikz}	
\usetikzlibrary{backgrounds,fit,decorations.pathreplacing}  % Tik\Z libraries

\usepackage[colorinlistoftodos,prependcaption]{todonotes}

\usepackage[ruled,vlined,noresetcount]{algorithm2e}
\SetKwInput{KwOutput}{Output}              % set the Output
\SetKwInOut{Input}{Input}
\SetKwInput{Output}{Output} 
\theoremstyle{plain}

\usepackage{tablefootnote}
\usepackage{footnotehyper}
\usepackage{lipsum}
\usepackage{amsfonts}
\usepackage{graphicx}
\usepackage{epstopdf}
\usepackage{algorithmic}
\usepackage{bm}
\usepackage{enumerate}
\usepackage{mathrsfs}
\ifpdf
  \DeclareGraphicsExtensions{.eps,.pdf,.png,.jpg}
\else
  \DeclareGraphicsExtensions{.eps}
\fi

\newtheorem{definitionenv}{Definition}
\newtheorem{lemmaenv}[definitionenv]{Lemma}
\newtheorem{theoremenv}[definitionenv]{Theorem}
\newtheorem{corollaryenv}[definitionenv]{Corollary}
\newtheorem{propositionenv}[definitionenv]{Proposition}
\newtheorem{conjectureenv}[definitionenv]{Conjecture}
\newtheorem{remarkenv}[definitionenv]{Remark}
\newenvironment{remark}{\begin{remarkenv}\rm}{\end{remarkenv}}
\newcommand{\br}{\begin{remark}}
	\newcommand{\er}{\end{remark}}

\newtheorem{exampleenv}{Example}
\newtheorem{app-lemmaenv}[section]{Lemma}

\newenvironment{definition}{\begin{definitionenv}\rm}{\end{definitionenv}}
\newenvironment{lemma}{\begin{lemmaenv}\rm}{\end{lemmaenv}}
\newenvironment{theorem}{\begin{theoremenv}\rm}{\end{theoremenv}}
\newenvironment{corollary}{\begin{corollaryenv}\rm}{\end{corollaryenv}}
\newenvironment{example}{\begin{exampleenv}\rm}{\end{exampleenv}}
\newenvironment{proposition}{\begin{propositionenv}\rm}{\end{propositionenv}}
\newenvironment{conjecture}{\begin{conjectureenv}\rm}{\end{conjectureenv}}
\newenvironment{app-lemma}{\begin{app-lemmaenv}\rm}{\end{app-lemmaenv}}

\newcommand{\bd}{\begin{definition}}
	\newcommand{\ed}{\end{definition}}
\newcommand{\bl}{\begin{lemma}}
	\newcommand{\el}{\end{lemma}}
\newcommand{\elp}{\hspace*{\fill} $\Box$
\end{lemma}}
\newcommand{\bt}{\begin{theorem}}
\newcommand{\et}{\end{theorem}}
\newcommand{\etp}{\hspace*{\fill} $\Box$
\end{theorem}}
\newcommand{\bc}{\begin{corollary}}
\newcommand{\ec}{\end{corollary}}
\newcommand{\ecp}{\hspace*{\fill} $\Box$
\end{corollary}}
\newcommand{\bcj}{\begin{conjecture}}
\newcommand{\ecj}{\end{conjecture}}

\newcommand{\be}{\begin{example}}
\newcommand{\ee}{\end{example}}
\newcommand{\eep}{\hspace*{\fill} $\Box$
\end{example}}
\newcommand{\bp}{\begin{proposition}}
\newcommand{\ep}{\end{proposition}}
\newcommand{\epp}{%\hspace*{\fill} $\Box$
\end{proposition}}

\newcommand{\eeq}{ \setcounter{equation} {\value{enumi}}}

\def\beq{\begin{equation}}
\def\eeq{\end{equation}}

\def\bean{\begin{IEEEeqnarray*}{rCl}}
\def\eean{\end{IEEEeqnarray*}}

\ifCLASSINFOpdf
\else
\fi
\begin{document}
\title{A Novel Loss Function-based Support Vector Machine for Binary Classification}

\author{Yan Li
           and Liping Zhang
\thanks{Yan Li and Liping Zhang are with the Department of Mathematical Sciences, Tsinghua University, Beijing 100084, China (email:li-yan20@mails.tsinghua.edu.cn; lipingzhang@tsinghua.edu.cn)}% <-this % stops a space
 }
\maketitle

\begin{abstract}
The previous support vector machine(SVM) including  $0/1$ loss SVM, hinge loss SVM, ramp loss SVM, truncated pinball loss  SVM, and others, overlooked the degree of penalty for the correctly classified samples within the margin. This oversight affects the generalization ability of the SVM classifier to some extent. To address this limitation, from the perspective of confidence margin, we propose a novel Slide loss function ($\ell_s$) to construct the support vector machine classifier($\ell_s$-SVM). By introducing the concept of proximal stationary point, and utilizing the property of Lipschitz continuity, we derive the first-order optimality conditions for  $\ell_s$-SVM. Based on this, we define the $\ell_s$ support vectors and working set of $\ell_s$-SVM. To efficiently handle $\ell_s$-SVM, we devise a fast alternating direction method of multipliers with the working set ($\ell_s$-ADMM), and provide the convergence analysis. The
numerical experiments on real world datasets confirm the robustness and effectiveness of the
proposed method.
\end{abstract}


% \fi
\IEEEpeerreviewmaketitle

\section{Introduction}
Support Vector Machine (SVM) has emerged as powerful and versatile tools in the domains of data mining, pattern recognization and machine learning, providing robust solutions to classification and regression problems. Introduced by Cortes and Vapnik~\cite{cortes1995support}, SVM has gained widespread popularity due to their ability to handle high-dimensional data, and generalization to unseen instances. At its essence, SVM is a supervised learning algorithm designed for both classification and regression tasks. Its primary goal is to find an optimal hyperplane that minimizes the classification errors on training data while maximizing the margin between them and obtain the better generalization ability. This hyperplane serves as a decision boundary, enabling the accurate predictions for new, unseen data points. SVM has been shown to be a formidable tool in addressing practical binary classification problems, in recent years, it has become one of the most used classification methods~\cite{cervantes2020comprehensive}. 

Given the training set $\{(\bm{x}_i,\bm{y}_i):i\in[m]\}\subseteq\mathbb{R}^n\times \{+1,-1\}$, where  $\bm{x}_i$ is the input feature vector and $\bm{y}_i$ denotes the corresponding output label. When the training samples can be linearly separated, that is, we  assume the existence of a hyperplane $\langle\bm{w},\bm{x}\rangle+b=0$ that perfect separates the training sample into two populations of positively and negatively labeled points, the pair $(\bm{w},b)$ returned by SVM is the solution of the following convex optimization problem 
\begin{equation}
    \min_{\bm{w}\in\mathbb{R}^n,b\in\mathbb{R}}\frac{1}{2}\Vert \bm{w}\Vert_2^2\quad 
\textup{s.t.}\ \bm{y}_i(\langle\bm{w},\bm{x}_i\rangle+b)\geq 1\ \forall i\in[m].
\end{equation}
In most practical settings, the training data is not linearly separable, which implies that for any hyperplane $\langle\bm{w},\bm{x}\rangle+b=0$, there exists sample $\bm{x}_i$ such that $\bm{y}_i(\langle\bm{w},\bm{x}_i\rangle+b)\ngeq 1$. This leads to the following general optimization defining SVM in the non-separable case :
\begin{equation}   \min_{\bm{w}\in\mathbb{R}^n,b\in\mathbb{R}}\frac{1}{2}\Vert \bm{w}\Vert_2^2+C\sum_{i=1}^{m}\ell(\bm{y}_i, f(\bm{x}_i))
\end{equation}
where $C>0$ represents a trade-off parameter, $\ell:\mathbb{R}\rightarrow\mathbb{R}$ stands for the loss function and $f(x):=\langle\bm{w},\bm{x}\rangle+b$. The first term $\frac{1}{2}\Vert \bm{w}\Vert_2^2$ is to maximize the margin and the second term controls the number of misclassification samples. The well known loss function is Heaviside step function, (or simply the $0/1$ loss):
\begin{equation*}
    \ell_{0/1}(t)=
    \begin{cases}
        1,\ t>0\\
        0,\ t\leq 0.
    \end{cases}
\end{equation*}
Specifically, there are 
\begin{itemize}
    \item the hard margin loss function~\cite{brooks2011support}~\cite{DBLP:books/daglib/0097035}:
    $$\ell(\bm{y}_i, f(\bm{x}_i))=\ell_{0/1}(1-\bm{y}_i f(\bm{x}_i)),$$
    \item the misclassification loss function~\cite{evgeniou2000regularization}~\cite{wang2020comprehensive}:
    $$\ell(\bm{y}_i, f(\bm{x}_i))=\ell_{0/1}(-\bm{y}_i f(\bm{x}_i))$$
\end{itemize}
 in the SVM classifier. Researchers have focused on developing other surrogate functions that are more tractable, since the non-convexity and discontinuity of $0/1$ loss make the problems hard to optimize. Notably one like 
hinge loss $\ell_h(t)=\max\{0, t\}$~\cite{cortes1995support}, while the convexity nature of which leads to the SVM classifier is sensitive to the presence of noises and outliers in training samples~\cite{yin2014fault}. To ameliorate the effectiveness of $\ell_h(t)$, other convex surrogates such as square hinge loss~\cite{zhang2001text}, huberized hinge loss~\cite{wang2008hybrid}, pinball loss~\cite{jumutc2013fixed}, $\epsilon$-insensitive pinball loss~\cite{liang2021support} are proposed, and the relevant solving methods on SVM classifier with the convex loss functions are researched, see e.g.,~\cite{yan2020efficient}~\cite{huang2016solution}~\cite{allen2018katyusha}~\cite{zhu2020support}~\cite{hsieh2008dual}~\cite{huang2015sequential}~\cite{wang2022safe}. To improve the situation that outliers play a leading role in determining the decision boundary, the truncated hinge loss~\cite{wu2007robust} $\ell_r(t)=\max\{0,\min\{\mu,t\}\}$ (ramp loss~\cite{brooks2011support} for $\mu=1$) is applied to solve the classification problem, which enhance the robustness to outliers. Other non-convex surrogates including rescaled hinge loss~\cite{xu2017robust,singla2020robust}, truncated pinball loss~\cite{shen2017support}, truncated least squares loss~\cite{chen2018sparse}, truncated logistic loss~\cite{park2011robust}, etc. have also attracted widespread attention to increase the generalization power of SVM, while the non-convexity of these loss functions bring the challenges in numerical computations. 
Recently, Wang et al.~\cite{wang2021support} proposed an efficient method to solve SVM with hard margin loss and develop the optimality theory under the assumption that the training samples obey the full column rank property, which is a meaningful attempt on the SVM classifier. 

Although the $0/1$ loss in SVM classifier quantifies the classification errors which essentially counts the number of misclassified samples or the samples falling within the margin, it does not explicitly consider the severity of these errors. Specifically, in the $0/1$ SVM classifier with the hard margin loss, samples that are correctly classified by the hyperplane $f(\bm{x})=\bm{0}$, satisfying $1>\bm{y}_if(\bm{x}_i)>0$, are penalized with a cost of $1$ even if the magnitude $\vert f(\bm{x}_i)\vert$ is sufficiently closing to $1$.
Similarly, in the $0/1$ SVM classifier with the misclassification loss, samples that are correctly classified by the hyperplane $f(\bm{x})=\bm{0}$ have a loss value of $0$, regardless of how close they are to the hyperplane $f(\bm{x})=\bm{0}$.  Therefore, the accuracy and efficiency of the SVM classifier with $0/1$ loss would be impacted to some extent. Other alternative loss functions, such as hinge loss, pinball loss, truncated least squares loss, truncated pinball loss, etc., also face a common issue:
 they do not applying the different degrees of penalization to distinguish the samples that are correctly classified but fall between the margin, including those near $f(\bm{x})=0$ and $f(\bm{x})=\pm{1}$.
 
Basing on above analysis, we give a new loss function of SVM classifier in view of the confidence margin~\cite{mohri2018foundations}. For any parameter $1>\epsilon, v>0$, we will define a Slide loss function, penalizes $f$ with the cost of $1$ when it misclassifies point $\bm{x}$ ($\bm{y}f(\bm{x})\leq 0$) and when it correctly classifies $\bm{x}$ with confidence no more than $1-v$ ($\bm{y}f(\bm{x})<1-v$), but also penalises $f$ (linearly) when it correctly classifies $\bm{x}$ with confidence no more than $1-\epsilon$ and more than $1-v$ ($1-v\leq\bm{y}f(\bm{x})<1-\epsilon$). Under the situation that the confidence of the sample $\bm{x}$ more than $1-\epsilon$, that is the sample is sufficiently close to anyone of the two classifier hyperplanes, it will not penalize $f$. We give the detail definition of Slide loss as follows: 
\begin{equation*}
    \ell_s(t):=
    \begin{cases}
        1\ &\textup{if}\ t>v\\
        \frac{t-\epsilon}{v-\epsilon}\ &\textup{if}\ v\geq t>\epsilon\\
        0\ &\textup{if}\ t\leq \epsilon
    \end{cases}
\end{equation*}
The Slide loss has some attractive properties. First, it has sparsity and robustness, which is benefit for weakening the impact from the outliers. Second, it consider the error degree and provides the varying degrees of penalization, when the samples are falling in the margin, and hence it enhances the generalization power of SVM classifier to some extent. Third, a key benefit of Slide loss as opposed to the $0/1$ loss is that it is $\frac{1}{v-\epsilon}$-Lipschitz, which is important to obtain the optimal theory. Moreover, it has a explicit expression of the limiting subdifferential and the proximal operator. 

In this paper, we formulate the robust binary SVM classifier as the following unconstrained optimization problem:
\begin{equation}\label{lspro}
    \min_{\bm{w},b}\frac{1}{2}\Vert \bm{w}\Vert^2+C\sum_{i=1}^m\ell_s(1-\bm{y}_i(\langle\bm{w},\bm{x}_i\rangle+b)),
\end{equation}
where $C$ is the penalty parameter. It can be abbreviated as $\ell_s$-SVM. The main contributions can be summarized as follows.:

\begin{itemize}
    \item Basing on the weakness of $0/1$ loss and other alternative loss functions, we propose a novel Slide loss ($\ell_s$) function, which allow us to present a new $\ell_s$-SVM classifier. We conducted an in-depth study on the subdifferential and proximal operator of the $\ell_s$ loss function. Based on these, we define the proximal stationary point of $\ell_s$-SVM and establish the optimality conditions. 
    \item Leveraging the aforementioned theoretical analysis, a precise definition of support vector is introduced, which is a small fraction of the entire training dataset. This geometric characteristic inspires us to devise a working set, and we integrate it with the ADMM algorithm to solve $\ell_s$-SVM, referred to as $\ell_s$-ADMM. This approach effectively reduces the computational cost per iteration, especially for large-scale datasets.
\end{itemize}
The rest of the paper is organized as follows. Section 2 gives the theoretical analysis of $\ell_s$ loss function, including the expression of subdifferential and proximal operator. The concept of proximal stationary point and the first order optimality conditions are given in Section 3. The whole framework of $\ell_s$-ADMM, which serve as the topic of the current paper, is explicitly studied in Section 4. In Section 5, the numerical experiments will be presented to highlight the robustness and effectiveness of $\ell_s$-SVM compared to the other six solvers.
\section{Theoretical analysis for $\ell_s$ loss function}
In this section, we conduct an in-depth study on the subdifferential and proximal operator of $\ell_s$ loss function. This research provides a solid theoretical foundation for establishing optimality conditions and the framework of algorithm in subsequent sections. To derive this, we give some necessary definitions.

\begin{definition}[Subgradient~\cite{rockafellar2009variational}]\label{subgzlp}
Let $f: \mathbb{R}^n\rightarrow \mathbb{R}\cup\{+\infty\}$  be a proper lower semicontinuous function and ${\rm{dom}}\;f\!:=\!\{\bm{x}\in\mathbb{R}^n:\!f(\bm{x})< +\infty\}$.
\begin{itemize}
  \item[\rm{(a)}] For each $\bm{x}\in {\rm{dom}}\;f$, the vector $\bm{v}\in\mathbb{R}^n$ is said to be a regular subgradient of $f$ at $\bm{x}$, written $\bm{v}\in\hat{\partial}f(\bm{x})$, if
  $$ f(\bm{y})\geq f(\bm{x})+\langle\bm{v},\bm{y}-\bm{x}\rangle+o(\Vert\bm{y}-\bm{x}\Vert).$$
 The set  $\hat{\partial}f(\bm{x})$ is called the regular subdifferential of $f$ at $\bm{x}$.
\item[\rm{(b)}] The vector $\bm{v}\in\mathbb{R}^n$ is said to be a (limiting) subgradient of $f$ at $\bm{x}\in{\rm{dom}}\;f$, written $\bm{v}\in\partial f(\bm{x})$, if there exists $\{\bm{x}^k\}\subset{\rm{dom}}\;f$ and $\{\bm{v}^k\}\subset\hat{\partial}f(\bm{x}^k)$ such that
      $$\bm{x}^k\to \bm{x},\quad  f(\bm{x}^k)\rightarrow f(\bm{x}),\quad \bm{v}^k\rightarrow\bm{v}, \quad \text{as $k\to\infty$}.$$
      The set $\partial f(\bm{x})$ is called the (limiting) subdifferential of $f$ at $\bm{x}$.
\end{itemize}
\end{definition}

The following proposition provides the explicit expression for the subdifferential of $\ell_s$ loss function.
\begin{proposition}
    Given $\epsilon$ and $v$, the subdifferential of the $\ell_s$ loss function $\ell_s$ at $t\in\mathbb{R}$ is:
\begin{equation}\label{subdif_slide}
        \partial\ell_s(t)=
        \begin{cases}
            0,\ &\textup{if}\;t>v\\
            \{0, \frac{1}{v-\epsilon}\},\ &\textup{if}\;t=v\\
            \frac{1}{v-\epsilon},\ &\textup{if}\;\epsilon<t<v\\
            [0, \frac{1}{v-\epsilon}],\ &\textup{if}\;t=\epsilon\\
            0,\ &\textup{if}\;t<\epsilon\\
        \end{cases}
    \end{equation}
\end{proposition}
\begin{proof}
    Clearly, $\ell_s$ loss function is non-differentiable only at $t=\epsilon$ and $t=v$. Based on this, we discuss the subdifferential of the  $\ell_s$ loss function in three cases: 
\begin{itemize}
    \item [(a)] When $t>v$, $t<\epsilon$, and $\epsilon<t<v$, the function $\ell_s$ is differentiable, and there exists a neighborhood of $t$ where it is smooth. Therefore, by the fact in ~\cite[Exercise 8.8]{rockafellar2009variational}, for $t>v$ or $t<\epsilon$, $\partial\ell_s(t)=\{0\}$; for $\epsilon<t<v$, $\partial\ell_s(t)=\{\frac{1}{v-\epsilon}\}$.
    \item[(b)] When $t=v$, using \Cref{subgzlp}, we have: 
    \begin{itemize}
        \item [(1)] If $t_k\rightarrow v^{+}$, then the regular subdifferential $\hat{\partial}\ell_s(t_k)=\{0\}$.
        \item[(2)] If $t_k\rightarrow v^{-}$, then the regular subdifferential $\hat{\partial}\ell_s(t_k)=\{\frac{1}{v-\epsilon}\}$.
        \item[(3)] If $t_k\rightarrow v$ and $t_k=v$, then the regular subdifferential $\hat{\partial}\ell_s(t_k)=\emptyset$.
    \end{itemize}
    Therefore, $\partial\ell_s(t_k)=\{0, \frac{1}{v-\epsilon}\}$.
    \item[(c)] When $t=\epsilon$, using \Cref{subgzlp}, we have:
    \begin{itemize}
        \item[(1)] If $t_k\rightarrow \epsilon^{+}$, then the regular subdifferential $\hat{\partial}\ell_s(t_k)=\{\frac{1}{v-\epsilon}\}$.
        \item[(2)] If $t_k\rightarrow \epsilon^{-}$, then the regular subdifferential $\hat{\partial}\ell_s(t_k)=\{0\}$.
        \item[(3)]If $t_k\rightarrow \epsilon$ and $t_k=\epsilon$, then the regular subdifferential $\hat{\partial}\ell_s(t_k)=[0,\frac{1}{v-\epsilon}]$.
    \end{itemize}
    Therefore, $\partial\ell_s(t_k)=[0, \frac{1}{v-\epsilon}]$.
\end{itemize}
  In conclusion, we provide the subdifferential of $\ell_s$ loss function as in \eqref{subdif_slide}.  
\end{proof}

The following proposition provides the explicit expression of the proximal operator for $\ell_s$ loss function.

\begin{proposition}
   For any given $\gamma$, $C$, and $s \in \mathbb{R}$. The proximal operator 
    \begin{equation*}
    \begin{aligned}
     \rm{Prox}_{\gamma C\ell_s}(s):&=\arg\min_t\{C\ell_s(t)+\frac{1}{2\gamma}(t-s)^2\}\\
     &=\arg\min_t\{\gamma C\ell_s(t)+\frac{1}{2}(t-s)^2\}
    \end{aligned}
\end{equation*}
admits a closed form as: 
    \begin{itemize}
        \item [(a)] for $0<\gamma C<2(v-\epsilon)^2$,\\
\begin{equation}\label{def_prox1}
    \textup{Prox}_{\gamma C \ell_{s}}(s)=
    \begin{cases}
       s\  &\textup{if}\ s>v+\frac{\gamma C}{2(v-\epsilon)} \\
       s\ \textup{or}\ s-\frac{\gamma C}{(v-\epsilon)}\  &\textup{if}\ s= v+\frac{\gamma C}{2(v-\epsilon)}\\
       s-\frac{\gamma C}{(v-\epsilon)}\  &\textup{if}\ \frac{\gamma C}{(v-\epsilon)}+\epsilon \leq s<v+\frac{\gamma C}{2(v-\epsilon)}  \\
       \epsilon\  &\textup{if}\ \epsilon<s<\frac{\gamma C}{(v-\epsilon)}+\epsilon \\
       s\ &\textup{if}\  s\leq \epsilon;
      \end{cases}
    \end{equation}
    \item [(b)] for $\gamma C\geq 2(v-\epsilon)^2$,\\
    \begin{equation}\label{solu_geq}
     \textup{Prox}_{\gamma C \ell_{s}}(s)=
     \begin{cases}
      s\ &\textup{if}\ s>\sqrt{2\gamma C}+\epsilon \\
      s\ \textup{or}\ \epsilon\ &\textup{if}\ s=\sqrt{2\gamma C}+\epsilon\\
      \epsilon\ &\textup{if}\ \epsilon<s<\sqrt{2\gamma C}+\epsilon\\
      s\ &\textup{if}\ s\leq \epsilon.
     \end{cases}
    \end{equation}
    \end{itemize}
\end{proposition}
\begin{proof}
Combining the definition of $\ell_s$ loss function, we can determine that $\textup{Prox}_{\gamma C\ell_s}(s)$ corresponds to the minimizer of the following piecewise function, denoted as $t^*$:
\begin{equation*}
\Phi(t):=
    \begin{cases}
        \phi_1(t)=\gamma C+\frac{1}{2}(t-s)^2\ \;&\textup{if}\;t>v\\
        \phi_2(t)=\gamma C+\frac{1}{2}(v-s)^2\ \;&\textup{if}\;t=v\\
        \phi_3(t)=\frac{\gamma C}{v-\epsilon}(t-\epsilon)+\frac{1}{2}(t-s)^2\ \;&\textup{if}\;\epsilon<t<v\\
        \phi_4(t)=\frac{1}{2}(\epsilon-s)^2\ \;&\textup{if}\;t=\epsilon\\
        \phi_5(t)=\frac{1}{2}(t-s)^2\ \;&\textup{if}\;t<\epsilon\\
    \end{cases}
\end{equation*}
For $i=1,2,3,4,5$,  let $\phi_i^*$ denote the minimum value of the function $\phi_i(t)$ and $t_i^*$ denote the corresponding point where the minimum is achieved. By simple calculations, we have: 
\begin{equation*}
    \begin{cases}
        \phi_1^*=\gamma C,\ &t_1^*=s\\
        \phi_2^*=\gamma C+\frac{1}{2}(v-s)^2,\ &t_2^*=v\\
        \phi_3^*=\frac{\gamma C}{v-\epsilon}(s-\epsilon)-\frac{1}{2}\left(\frac{\gamma C}{v-\epsilon}\right)^2,\ &t_3^*=s-\frac{\gamma C}{v-\epsilon}\\
        \phi_4^*=\frac{1}{2}(s-\epsilon)^2,\ &t_4^*=\epsilon\\
        \phi_5^*=0,\ &t_5^*=s.
    \end{cases}
\end{equation*}
Now we proceed with the discussion in three cases: 
\begin{itemize}
    \item [(i)]When $\gamma C<2(v-\epsilon)^2$:
    \begin{itemize}
        \item [(1)] If $s>v+\frac{\gamma C}{2(v-\epsilon)}$, then $\min\{\phi_2^*, \phi_3^*, \phi_4^*, \phi_5^*\}>\phi_1^*$, hence $t^*=s$.
        \item[(2)] If$s=v+\frac{\gamma C}{2(v-\epsilon)}$, then $\min\{\phi_2^*, \phi_4^*, \phi_5^*\}>\phi_1^*=\phi_3^*$, hence $t^*=s$or$s-\frac{\gamma C}{v-\epsilon}$.
        \item[(3)]If $\sqrt{2\gamma C}+\epsilon<s<v+\frac{\gamma C}{2(v-\epsilon)}$, then $\min\{\phi_1^*, \phi_2^*, \phi_4^*, \phi_5^*\}>\phi_3^*$, hence $t^*=s-\frac{\gamma C}{v-\epsilon}$.
        \item[(4)]If $s=\sqrt{2\gamma C}+\epsilon$, then $\min\{\phi_1^*, \phi_2^*, \phi_4^*, \phi_5^*\}>\phi_3^*$, hence $t^*=s-\frac{\gamma C}{v-\epsilon}$.
        \item[(5)]If $\frac{\gamma C}{v-\epsilon}+\epsilon<s<\sqrt{2\gamma C}+\epsilon$, then $\min\{\phi_1^*, \phi_2^*, \phi_4^*, \phi_5^*\}>\phi_3^*$, hence $t^*=s-\frac{\gamma C}{v-\epsilon}$.
        \item[(6)]If $s=\frac{\gamma C}{v-\epsilon}+\epsilon$, then $\min\{\phi_1^*, \phi_2^*, \phi_3^*, \phi_5^*\}>\phi_4^*$, hence $t^*=\epsilon$.
        \item[(7)]If $\epsilon<s<\frac{\gamma C}{v-\epsilon}+\epsilon$, then $\min\{\phi_1^*, \phi_2^*, \phi_3^*, \phi_5^*\}>\phi_4^*$, hence $t^*=\epsilon$.
        \item[(8)]If $s=\epsilon$, then $\min\{\phi_1^*, \phi_2^*, \phi_3^*\}>\phi_4^*=\phi_5^*$, hence $t^*=\epsilon$.
        \item[(9)]If $s<\epsilon$, then $\min\{\phi_1^*, \phi_2^*, \phi_3^*, \phi_4^*\}>\phi_5^*$, hence $t^*=s$.
    \end{itemize}
    \item[(ii)]When $\gamma C>2(v-\epsilon)^2$:
    \begin{itemize}
        \item [(1)]If $s>v+\frac{\gamma C}{2(v-\epsilon)}$, then $\min\{\phi_2^*, \phi_3^*, \phi_4^*, \phi_5^*\}>\phi_1^*$, hence $t^*=s$.
        \item[(2)]If $s=v+\frac{\gamma C}{2(v-\epsilon)}$, then $\min\{\phi_2^*, \phi_3^*, \phi_4^*, \phi_5^*\}>\phi_1^*$, hence $t^*=s$.
        \item[(3)]If $\sqrt{2\gamma C}+\epsilon<s<v+\frac{\gamma C}{2(v-\epsilon)}$, then $\min\{\phi_2^*, \phi_3^*, \phi_4^*, \phi_5^*\}>\phi_1^*$, hence $t^*=s$.
        \item[(4)] If $s=\sqrt{2\gamma C}+\epsilon$, then $\min\{\phi_2^*, \phi_3^*, \phi_5^*\}>\phi_1^*=\phi_4^*$, hence $t^*=s$ or $\epsilon$.
        \item[(5)] If $\epsilon<s<\sqrt{2\gamma C}+\epsilon$, then $\min\{\phi_1^*, \phi_2^*, \phi_3^*, \phi_5^*\}>\phi_4^*$, hence $t^*=\epsilon$.
        \item[(6)] If $s=\epsilon$, then $\min\{\phi_1^*, \phi_2^*, \phi_3^*\}>\phi_4^*=\phi_5^*$, hence $t^*=s=\epsilon$.
        \item [(7)]If $s<\epsilon$, then $\min\{\phi_1^*, \phi_2^*, \phi_3^*, \phi_4^*\}>\phi_5^*$, hence $t^*=s$.
    \end{itemize}
    \item[(iii)]When $\gamma C=2(v-\epsilon)^2$:
    \begin{itemize}
        \item [(1)]If $s>\sqrt{2\gamma C}+\epsilon$, then $\min\{ \phi_2^*, \phi_3^*, \phi_4^*, \phi_5^*\}>\phi_1^*$, hence $t^*=s$.
        \item[(2)]If $s=\sqrt{2\gamma C}+\epsilon$, then $\min\{ \phi_2^*, \phi_3^*, \phi_5^*\}>\phi_1^*= \phi_4^*$, hence $t^*=s$ or $\epsilon$.
        \item[(3)]If $\epsilon<s<\sqrt{2\gamma C}+\epsilon$, then $\min\{\phi_1^*, \phi_2^*, \phi_3^*, \phi_5^*\}>\phi_4^*$, hence $t^*=\epsilon$.
        \item[(4)]If $s=\epsilon$, then
      $\min\{\phi_1^*, \phi_2^*, \phi_3^*\}>\phi_4^*=\phi_5^*$, hence $t^*=\epsilon$.
        \item[(5)]If $s<\epsilon$, then
        $\min\{\phi_1^*, \phi_2^*, \phi_3^*, \phi_4^*\}>\phi_5^*$, hence $t^*=s$.
    \end{itemize}
\end{itemize}
In summary, we can derive the proximal operator for $\ell_s$ loss function as given in \eqref{def_prox1} and \eqref{solu_geq}.
\end{proof}

\section{Optimality conditions for $\ell_s$-SVM}
To facilitate subsequent analysis, we define the following notation.  Define $[m]:=\{1,2,\cdots,m\}$, $A:=[y_1\bm{x}_1 y_2\bm{x}_2 \cdots y_m\bm{x}_m]^{\intercal}$,
    $\bm{y}:=(y_1, y_2,\cdots, y_m)^{\intercal}$, $B:=[A\; \bm{y}]$,
    $\bm{1}:=(1, 1,\cdots,1)^{\intercal}$ and
\begin{equation*}
  \mathcal{L}_s(u)=\sum_{i=1}^{m}\ell_{s}(\bm{u}_i)=\sum_{i=1}^{m}\min\{1,\max\{\frac{\bm{u}_i-\epsilon}{v-\epsilon},0\}\},
\end{equation*}
Furthermore, for any finite set of indices $\Omega \subseteq [m]$, $\Omega_c$ represents the complement of $\Omega$. We define $\bm{x}_{\Omega} \in \mathbb{R}^{\vert \Omega\vert}$ as the $\vert \Omega\vert$-dimensional subvector of $\bm{x}$, where the components indexed by $\Omega$ are the same as those of $\bm{x}$; $A_{\Omega} \in \mathbb{R}^{\vert\Omega\vert \times n}$ is defined as the submatrix of $A$, where the row vectors indexed by $\Omega$ are the same as those of $A$.

Using the notation introduced above, we can rewrite \eqref{lspro} as:
\begin{equation}\label{orig_pro}
    \min_{\bm{w},b,\bm{u}}\frac{1}{2}\Vert \bm{w}\Vert_2^2+C\mathcal{L}_s(\bm{u})\quad s.t.\ \bm{u}+A\bm{w}+b\bm{y}=\bm{1},
\end{equation}
the augmented Lagrangian function of which is defined as follows:
\begin{equation*}
   L(\bm{w}, b,\bm{u},\bm{\lambda}):= \frac{1}{2}\Vert \bm{w}\Vert_2^2+C\mathcal{L}_s(\bm{u})+\langle \bm{\lambda},\bm{u}+A\bm{w}+b\bm{y}-\bm{1}\rangle+\frac{1}{2\gamma}\Vert \bm{u}+A\bm{w}+b\bm{y}-\bm{1}\Vert^2, 
   \end{equation*}
   where $\gamma>0$ is the penalty parameter.
 In the following, we present a new definition of stationary point derived from the augmented Lagrangian function:
\begin{definition}
We say $(\bm{w}^*;\bm{b}^*;\bm{u}^*)$ is a proximal stationary point of \eqref{orig_pro} if there is a Lagrangian multiplier $\bm{\lambda}^*$ and a constant $\gamma>0$ such that 
\begin{equation}\label{def_sta}
\begin{cases}
    \bm{w}^*+A^{\top}\bm{\lambda}^*=0\\
   \langle\bm{y},\bm{\lambda}^*\rangle=0\\
   \bm{u}^*+A\bm{w}^*+b^*\bm{y}=1\\
   \bm{u}^*\in\textup{Prox}_{\gamma C\mathcal{L}_s}(\bm{u}^*-\gamma\bm{\lambda}^*).
\end{cases}
\end{equation}
\end{definition}
According to the definition of the proximal operator $\textup{Prox}_{\gamma C\mathcal{L}_s}$:
\begin{equation*}
   \textup{Prox}_{\gamma C \mathcal{L}_s}(\bm{s}):=\arg\min_{\bm{x}}\{\gamma C\mathcal{L}_s(\bm{x})+\frac{1}{2}\Vert\bm{x}-\bm{s}\Vert_2^2\} =
   \begin{bmatrix}
     \textup{Prox}_{\gamma C \ell_{s}}(\bm{s}_1)\\
     \textup{Prox}_{\gamma C \ell_{s}}(\bm{s}_2)\\
     \vdots\\
     \textup{Prox}_{\gamma C \ell_{s}}(\bm{s}_m)
   \end{bmatrix}.
\end{equation*} 
In the previous section, we have already provided the explicit solution for the proximal operator of $\ell_s$ loss function. Therefore, it is straightforward to verify whether \eqref{def_sta} holds.

To elucidate the relationship between the proximal stationary point and the local minimizer of problem \eqref{orig_pro}, we first introduce some index sets and the fixed parameters. 
For a point $(\bm{w}^*;\bm{b}^*;\bm{u}^*)$, let's define the index sets 
\begin{equation*}
    \begin{aligned}
        &S^*:=\{i\ |\ \bm{u}_i^*>v\},\quad E^*:=\{i\ |\ \bm{u}_i^*=v\},\\
        &T^*:=\{i\ |\ \epsilon<\bm{u}^*_i<v\},\quad I^*:=\{i\ |\ \bm{u}^*_i=\epsilon\},\quad O^*:=\{i\ |\ \bm{u}^*_i<\epsilon\}
    \end{aligned}
\end{equation*} 
and the constant parameters
\begin{equation*}
\begin{aligned}
&\gamma_1^*:=
    \begin{cases}
        \min\frac{2(\bm{u}^*_i-v)(v-\epsilon)}{C},\ i\in S^*,\\
        \infty,\ S^*=\emptyset,
    \end{cases}
  &\gamma_2^*:=
    \begin{cases}
     \min\frac{2(v-\bm{u}^*_i)(v-\epsilon)}{C},\ i\in T^*,\\
     \infty,\ T^*=\emptyset,
    \end{cases}\\
  &\gamma_3^*:= 
  \begin{cases}
      \frac{2(v-\epsilon)^2}{C},\ i\in I^*,\\
      \infty,\ I^*=\emptyset,
  \end{cases}
  &\gamma_4^*:=
  \begin{cases}
      \min\{\frac{(\bm{u}^*_i-\epsilon)^2}{2C}:\bm{u}^*_i>\epsilon\},\ \bm{u}^*_i>\epsilon\\
      \infty,\ \textup{otherwise}
  \end{cases}
    \end{aligned}
\end{equation*}
Based on the above notation, we present the first-order necessary and first-order sufficient conditions for problem \eqref{orig_pro}.
\begin{theorem}\label{suf_nec_con}
The relationship between the proximal stationary point and the local minimizer of problem \eqref{orig_pro} is as follows:
\begin{itemize}
\item[(i)]  A local minimizer $(\bm{w}^*;\bm{b}^*;\bm{u}^*)$ of  \eqref{orig_pro}  is a proximal stationary point in terms of $0<\gamma\leq\gamma^*:=\min\{\gamma_1^*,\gamma_2^*,\gamma_3^*,\gamma_4^*\}$ if $E^*=\emptyset$.
    \item[(ii)]If $(\bm{w}^*,\bm{b}^*,\bm{u}^*)$ with $\gamma>0$ is a proximal stationary point, then it is a local minimizer of  \eqref{orig_pro} , and $\bm{\lambda}^*=(\bm{\lambda}^*_1,\bm{\lambda}^*_2,\cdots,\bm{\lambda}^*_m)^{\top}$ satisfies 
\begin{equation}\label{lam_range}
\begin{cases}
    \bm{\lambda}^*_i\in[-\frac{C}{v},0]\quad &\textup{if}\ 0<\gamma C<2v^2\\
   \bm{\lambda}^*_i\in[-\sqrt{\frac{2C}{\gamma}},0]\quad &\textup{if}\ \gamma C\geq 2v^2
\end{cases}
\end{equation}
for $i\in\mathbb{N}$.
\end{itemize}
\end{theorem}
\begin{proof}
We first prove that (i) holds. For ease of expression, let $\bm{z}:=[\bm{w}; b]$, $h(\bm{z}):=\frac{1}{2}\Vert \bm{w}\Vert^2$. According to ~\cite[Theorem 10.1]{rockafellar2009variational}, if $(\bm{w}^*,\bm{b}^*,\bm{u}^*)$ is the local minimizer of problem \eqref{orig_pro}, then we have
$$\bm{0}\in\partial\left(h(\bm{z}^*)+C\mathcal{L}_s(\bm{1}-B\bm{z}^*)\right).$$ Since $\ell_s$ loss function is Lipschitz continuous, according to ~\cite[Theorem 10.6]{rockafellar2009variational} and ~\cite[Theorem 9.13]{rockafellar2009variational}, we have $$\bm{0}\in\nabla h(\bm{z}^*)-CB^{\top}\partial\mathcal{L}_s(\bm{u}^*
),$$
where $\bm{u}^*=1-A\bm{w}^*-b^*\bm{y}$. This implies the existence of $-{\bm{\lambda}}^*\in C\partial\mathcal{L}_s(\bm{u}^*)$ such that $\bm{0}=\nabla h(\bm{z}^*)+B^{\top}{\bm{\lambda}}^*$. Combining the above results, we obtain the following system:  
\begin{equation*}
    \begin{cases}
        \bm{w}^*+A^{\top}\bm{\lambda}^*=\bm{0},\\
        \langle\bm{y},\bm{\lambda}^*\rangle=0,\\
        \bm{1}-A\bm{w}^*-b^*\bm{y}=\bm{u}^*,\\
        0\in \bm{\lambda}^*+C\partial\mathcal{L}_s(\bm{u}^*).
    \end{cases}
\end{equation*}
Therefore, to establish (\ref{def_sta}), it is necessary to prove that $0\in \bm{\lambda}^*+C\partial\mathcal{L}_s(\bm{u}^*)$ implies $\bm{u}^*_i\in\textup{Prox}_{\gamma C\ell_s}(\bm{u}^*_i-\gamma\bm{\lambda}^*_i)$ for $i\in[m]$ with $0<\gamma\leq\gamma^*$. 
Combining the Lipschitz continuity of $\ell_s$ loss function and ~\cite[Proposition 10.5]{rockafellar2009variational}, we obtain $$\partial\mathcal{L}_s(\bm{u}^*)=\partial\ell_s(u_1^*)\times \cdots\times \partial\ell_s(u_m^*),$$
and consequently, based on the explicit expression of the subdifferential of $\ell_s$ loss function provided in the previous section, $\bm{\lambda}^*$ can be represented as follows:
\begin{equation}
    \bm{\lambda}^*_i\in
    \begin{cases}
        0,\ &\textup{for}\ \bm{u}^*_i>v,\\
        \{\frac{-C}{v-\epsilon}, 0\}, \ &\textup{for}\ \bm{u}^*_i=v,\\
        \frac{-C}{v-\epsilon},\ &\textup{for}\ \epsilon<\bm{u}^*_i<v,\\
        [\frac{-C}{v-\epsilon},0],\ &\textup{for}\ \bm{u}^*_i=\epsilon,\\
        0,\ &\textup{for}\ \bm{u}^*_i<\epsilon.
    \end{cases}
\end{equation}
In the following, the cases where $0<\gamma C<(v-\epsilon)^2$ and $\gamma C\geq (v-\epsilon)^2$ need to be considered separately.

Case I: For $0<\gamma C<(v-\epsilon)^2$. 
\begin{itemize}
    \item [(a)] As $i\in S^*$, we obtain that $\bm{u}^*_i>v$ and $\bm{\lambda}^*_i=0$, which implies $\bm{s}^*_i:=\bm{u}^*_i-\gamma\bm{\lambda}^*_i=\bm{u}^*_i$. Then the fact that $\gamma\leq\gamma^*_1$ gives that 
    \begin{equation*}
   \gamma\leq\frac{2(\bm{u}^*_i-v)(v-\epsilon)}{C}=\frac{2(\bm{s}^*_i-v)(v-\epsilon)}{C},
    \end{equation*}
    and hence $\bm{s}^*_i\geq v+\frac{\gamma C}{2(v-\epsilon)}$.
    \item[(b)] As $i\in T^*$, we obtain that $\epsilon <\bm{u}^*_i<v$ and $\bm{\lambda^*}_i=\frac{-C}{v-\epsilon}<0$, which implies $\bm{s}^*_i:=\bm{u}^*_i-\gamma\bm{\lambda}^*_i=\bm{u}^*_i+\frac{\gamma C}{v-\epsilon}>\epsilon+\frac{\gamma C}{v-\epsilon}$. Moreover, the fact $\gamma\leq\gamma^*_2$ yields $\bm{u}^*_i+\frac{\gamma C}{v-\epsilon}\leq v+\frac{\gamma C}{2(v-\epsilon)}$, that is $\bm{s}^*_i\leq v+\frac{\gamma C}{2(v-\epsilon)}$. Hence $\epsilon+\frac{\gamma C}{v-\epsilon}<\bm{s}^*_i\leq v+\frac{\gamma C}{2(v-\epsilon)}$.
    \item [(c)] As $i\in I^*$, we obtain that $\bm{u}^*_i=\epsilon$ and $\bm{\lambda}^*_i\in[\frac{-C}{v-\epsilon},0]$, which implies that $\bm{s}^*_i:=\bm{u}^*_i-\gamma\bm{\lambda}^*_i=\epsilon-\gamma\bm{\lambda}^*_i\in[\epsilon,\frac{\gamma C}{v-\epsilon}]$.
    \item [(d)] As $i\in O^*$, we obtain that $\bm{u}^*_i<\epsilon$ and $\bm{\lambda}^*_i=0$, which yields $\bm{s}^*_i:=\bm{u}^*_i-\gamma\bm{\lambda}^*_i=\bm{u}^*_i<\epsilon.$
\end{itemize}
The above analysis, in conjunction with the expression in (\ref{def_prox1}), establishes that $\bm{u}^*_i\in\textup{Prox}_{\gamma C\ell_s}(\bm{s}^*_i)$ for $i\in[m]$.

Case II: For $\gamma C\geq 2(v-\epsilon)^2$.
\begin{itemize}
    \item [(a)] As $i\in E^*\cup T^*\cup S^*$, we obtain that $\bm{u}^*_i>\epsilon$. The fact that $\gamma\leq\gamma^*_4$ yields $\bm{u}^*_i\geq \sqrt{2\gamma^*_4C}+\epsilon\geq \sqrt{2\gamma^*C}+\epsilon\geq2(v-\epsilon)+\epsilon>v$, which implies $\bm{\lambda}_i=0$. Hence $\bm{s}^*_i:=\bm{u}^*_i-\gamma\bm{\lambda}^*_i=\bm{u}^*_i\geq \sqrt{2\gamma^*C}+\epsilon$.
    \item [(b)] As $i\in I^*$, we obtain that $\bm{u}^*_i=\epsilon$ and $\bm{\lambda}^*_i\in[\frac{-C}{v-\epsilon}, 0]$. The fact that $\gamma\leq\gamma^*_3$ yields $\epsilon\leq \bm{s}^*_i:=\bm{u}^*_i-\gamma\bm{\lambda}^*_i\leq \epsilon+2(v-\epsilon)\leq \epsilon+\sqrt{2\gamma C}$.
    \item[(c)] As $i\in O^*$, we obtain that $\bm{u}^*_i<\epsilon$ and $\bm{\lambda}^*_i=0$, and hence $\bm{s}^*_i:=\bm{u}^*_i-\gamma\bm{\lambda}^*_i<\epsilon$.
\end{itemize}
The above discussion combined with the expression in (\ref{solu_geq}) show that $\bm{u}^*_i\in\textup{Prox}_{\gamma C\ell_s}(\bm{s}^*_i)$ for $i\in[m]$.

Next, we prove that (ii) holds. Define $\Lambda:=\{(\bm{w};b;\bm{u})\ \vert\ \bm{u}+Aw+b\bm{y}=\bm{1}\}$. Firstly, it is easy to get for any $(\bm{w};b;\bm{u})\in\Lambda$
\begin{equation}\label{term1_ineq}
    \begin{aligned}
    \Vert \bm{w}\Vert^2-\Vert \bm{w}^*\Vert^2&\geq 2\langle \bm{w}-\bm{w}^*,\bm{w}^*\rangle\\
    &=-2\langle A(\bm{w}-\bm{w}^*),\bm{\lambda}^*\rangle\\
    &=2\langle \bm{u}-\bm{u}^*,\bm{\lambda}^*\rangle +2(b-b^*)\langle y,\bm{\lambda}^*\rangle \\
    &=2\langle \bm{u}-\bm{u}^*,\bm{\lambda}^*\rangle.
    \end{aligned}
\end{equation}
Denote $\delta:=\begin{cases}
    \frac{\gamma C}{2(v-\epsilon)}\quad &if\ 0<\gamma C<2(v-\epsilon)^2\\
    v-\epsilon\quad &if\ \gamma C\geq 2(v-\epsilon)^2
\end{cases}$ and $\delta_m:=\frac{\delta}{\sqrt{2m}}$. Define 
\begin{equation*}
 \mathcal{U}((\bm{w}^*;b^*;\bm{u}^*),\delta):=\{(\bm{w};b;\bm{u})\ \vert\ \Vert(w;b)-(w^*;b^*)\Vert\leq\frac{\delta}{\sqrt{2}},\ \vert \bm{u}_i-u^*_i\vert\leq\delta_m\}   
\end{equation*}
In the sequel, we  will show that 
\begin{equation*}
    \frac{1}{2}\Vert \bm{w}^*\Vert^2+C\mathcal{L}_s(\bm{u}^*)\leq\frac{1}{2}\Vert \bm{w}\Vert^2+C\mathcal{L}_s(\bm{u})\quad \forall\ (\bm{w};b;\bm{u})\in\mathcal{U}((\bm{w}^*;b^*;\bm{u}^*),\delta)\cap\Lambda,
\end{equation*}
which further implies $(\bm{w}^*;b^*;\bm{u}^*)$ is a local minimizer. In fact, it suffice to show 
\begin{equation}\label{sim_inequ}
    C\mathcal{L}_s(\bm{u})-C\mathcal{L}_s(\bm{u}^*)+\langle \bm{u}-\bm{u}^*,\bm{\lambda}^*\rangle\geq 0 \quad \forall\ (\bm{w};b;\bm{u})\in\mathcal{U}((\bm{w}^*;b^*;\bm{u}^*),\delta)\cap\Lambda.
\end{equation}
Case I: For $0<\gamma C<2(v-\epsilon)^2$.
Define $\bm{s}^*:=\bm{u}^*-\gamma\bm{\lambda}^*$ and 
\begin{equation}\label{def_Gamma}
    \begin{aligned}
    &\Gamma^*_1:=\{i\in\mathbb{N}\ \vert\ \bm{s}^*_i\leq \epsilon\};\\
    &\Gamma^*_2:=\{i\in\mathbb{N}\ \vert\ \epsilon<\bm{s}^*_i<\frac{\gamma C}{v-\epsilon}+\epsilon\};\\
   &\Gamma^*_3:=\{i\in\mathbb{N}\ \vert\ \frac{\gamma C}{v-\epsilon}+\epsilon \leq\bm{s}^*_i<v+\frac{\gamma C}{2(v-\epsilon)}\}\cup\{i\in\mathbb{N}\ \vert\ \bm{s}^*_i=v+\frac{\gamma C}{2(v-\epsilon)}, \bm{\lambda}^*_i\neq 0\};\\
   &\Gamma^*_4:=\{i\in\mathbb{N}\ \vert\ \bm{s}^*_i>v+\frac{\gamma C}{2(v-\epsilon)}\}\cup\{i\in\mathbb{N}\ \vert\ \bm{s}^*_i=v+\frac{\gamma C}{2(v-\epsilon)}, \bm{\lambda}^*_i= 0\}.
    \end{aligned}
\end{equation}
By the closed solution in (\ref{def_prox1}) and relation in (\ref{def_sta}), we have
\begin{equation*}
    \begin{aligned}
    &\bm{u}^*_{\Gamma^*_1}=(\textup{Prox}_{\gamma C\mathcal{L}_s}(\bm{u}^*-\gamma\bm{\lambda}^*))_{\Gamma^*_1}=(\bm{u}^*-\gamma\bm{\lambda}^*)_{\Gamma^*_1};\\
    &\bm{u}^*_{\Gamma^*_2}=(\textup{Prox}_{\gamma C\mathcal{L}_s}(\bm{u}^*-\gamma\bm{\lambda}^*))_{\Gamma^*_2}=\epsilon;\\
    &\bm{u}^*_{\Gamma^*_3}=(\textup{Prox}_{\gamma C\mathcal{L}_s}(\bm{u}^*-\gamma\bm{\lambda}^*))_{\Gamma^*_3}=(\bm{u}^*-\gamma\bm{\lambda}^*-\frac{\gamma C}{v-\epsilon}\bm{1})_{\Gamma^*_3};\\
     &\bm{u}^*_{\Gamma^*_4}=(\textup{Prox}_{\gamma C\mathcal{L}_s}(\bm{u}^*-\gamma\bm{\lambda}^*))_{\Gamma^*_4}=(\bm{u}^*-\gamma\bm{\lambda}^*)_{\Gamma^*_4},
    \end{aligned}
\end{equation*}
which  implies 
\begin{equation*}
    \begin{cases}
        &\bm{\lambda}^*_{\Gamma_1^*}=0;\\
        &\bm{u}^*_{\Gamma_2^*}=\epsilon;\\
        &\bm{\lambda}^*_{\Gamma_3^*}=-\frac{C}{v-\epsilon}\bm{1}_{\Gamma_3^*};\\
        &\bm{\lambda}^*_{\Gamma_4^*}=0.
    \end{cases}
\end{equation*}
Combining with (\ref{def_Gamma}), it yields that
\begin{equation}\label{0_2v}
    \begin{cases}
        \bm{\lambda}^*_i=0,\quad \bm{u}^*_i\leq \epsilon,\quad i\in\Gamma_1^*;\\
        -\frac{C}{v-\epsilon}<\bm{\lambda}^*_i<0,\quad \bm{u}_i^*=\epsilon,\quad i\in\Gamma_2^*;\\
        \bm{\lambda}^*_i=-\frac{C}{v-\epsilon},\quad \epsilon\leq \bm{u}_i^*\leq v-\frac{\gamma C}{2(v-\epsilon)},\quad i\in\Gamma_3^*;\\
        \bm{\lambda}^*_i=0,\quad \bm{u}^*_i\geq v+\frac{\gamma C}{2(v-\epsilon)},\quad i\in\Gamma^*_4,
    \end{cases}
\end{equation}
Hence $-\frac{C}{v-\epsilon}\leq\bm{\lambda}^*_i\leq 0$ for $0<\gamma C<2(v-\epsilon)^2$.

Define $\mathring{\Gamma}:=\Gamma_2^*\cup\Gamma^*_3$ and $\mathring{\Gamma}_c:=\Gamma^*_1\cup\Gamma^*_4$. We will present 
\begin{equation*}
\begin{aligned}
    C\mathcal{L}_s(\bm{u}_{\mathring{\Gamma}})-C\mathcal{L}_s(\bm{u}_{\mathring{\Gamma}}^*)+\langle \bm{u}_{\mathring{\Gamma}}-\bm{u}_{\mathring{\Gamma}}^*,\bm{\lambda}^*_{\mathring{\Gamma}}\rangle\geq 0\quad \textup{and}\quad
    C\mathcal{L}_s(\bm{u}_{\mathring{\Gamma}_c})-C\mathcal{L}_s(\bm{u}_{\mathring{\Gamma}_c}^*)\geq 0.
\end{aligned}    
\end{equation*}
Since $\epsilon\leq \bm{u}^*_i\leq v-\frac{\gamma C}{v-\epsilon}$ for $i\in\Gamma_3^*$, we have 
\begin{equation*}
    \bm{u}^*_i-\delta_m\leq \bm{u}_i\leq \bm{u}^*_i+\delta_m<v
\end{equation*}
for any $\bm{u}_i$ satisfying $\vert \bm{u}_i-\bm{u}^*_i\vert\leq \delta_m$, and then $\ell_{s}(\bm{u}_i)\geq \frac{\bm{u}_i-\epsilon}{v-\epsilon}$ and $\ell_{s}(\bm{u}_i^*)=\frac{\bm{u}_i^*-\epsilon}{v-\epsilon}$.
Therefore, 
\begin{equation*}
\begin{aligned}
    &C\ell_{s}(\bm{u}_i)-C\ell_{s}(\bm{u}_i^*)+\bm{\lambda}^*_i(\bm{u}_i-\bm{u}^*_i)\\
    \geq&C(\frac{\bm{u}_i}{v-\epsilon}-\frac{\bm{u}^*_i}{v-\epsilon})+\bm{\lambda}^*_i(\bm{u}_i-\bm{u}^*_i)\\
    =&(\bm{u}_i-\bm{u}^*_i)(\frac{C}{v-\epsilon}+\bm{\lambda}^*_i)=0
\end{aligned}
\end{equation*}
Since $\bm{u}_i^*=\epsilon$ for $i\in\Gamma_2^*$,
we have $ \epsilon-\delta_m\leq \bm{u}_i\leq \epsilon+\delta_m$ for any $\bm{u}_i$ satisfying $\vert \bm{u}_i-\bm{u}^*_i\vert\leq \delta_m$. If $\epsilon\leq \bm{u}_i\leq \epsilon+\delta_m$, we can construct that
\begin{equation*}
\begin{aligned}
    &C\ell_{s}(\bm{u}_i)-C\ell_{s}(\bm{u}_i^*)+\bm{\lambda}^*_i(\bm{u}_i-\bm{u}^*_i)\\
    \geq&C(\frac{\bm{u}_i}{v-\epsilon}-\frac{\bm{u}^*_i}{v-\epsilon})+\bm{\lambda}^*_i(\bm{u}_i-\bm{u}^*_i)\\
    =&(\bm{u}_i-\bm{u}^*_i)(\frac{C}{v-\epsilon}+\bm{\lambda}^*_i)\geq 0.
\end{aligned}
\end{equation*}
If $\epsilon-\delta_m\leq \bm{u}_i< \epsilon$, we can construct that
\begin{equation*}
\begin{aligned}
    &C\ell_{s}(\bm{u}_i)-C\ell_{s}(\bm{u}_i^*)+\bm{\lambda}^*_i(\bm{u}_i-\bm{u}^*_i)\\
    =&0+\bm{\lambda}^*_i(\bm{u}_i-\epsilon)>0.
\end{aligned}
\end{equation*}
Hence $C\mathcal{L}_s(\bm{u}_{\mathring{\Gamma}})-C\mathcal{L}_s(\bm{u}_{\mathring{\Gamma}}^*)+\langle \bm{u}_{\mathring{\Gamma}}-\bm{u}_{\mathring{\Gamma}}^*,\bm{\lambda}^*_{\mathring{\Gamma}}\rangle=\sum_{i\in{\mathring{\Gamma}}}C\ell_{s}(\bm{u}_i)-C\ell_{s}(\bm{u}_i^*)+\bm{\lambda}^*_i(\bm{u}_i-\bm{u}^*_i) \geq 0$.

Since $\bm{u}^*_i\leq \epsilon$ for $i\in\Gamma^*_1$, we have $\bm{u}_i\leq \bm{u}^*_i+\delta_m<v$ for any $\bm{u}_i$ satisfying $\vert \bm{u}_i-\bm{u}^*_i\vert\leq \delta_m$, and then $C\ell_{s}(\bm{u}_i)\geq C\ell_{s}(\bm{u}^*_i)=0$. 
Since $\bm{u}^*_i\geq v+\frac{\gamma C}{2(v-\epsilon)}$ for $i\in\Gamma^*_4$, we have $\bm{u}_i\geq \bm{u}^*_i-\delta_m\geq v$ for any $\bm{u}_i$ satisfying $\vert \bm{u}_i-\bm{u}^*_i\vert\leq \delta_m$, and then $C\ell_{s}(\bm{u}_i)= C\ell_{s}(\bm{u}^*_i)=C$. 
Hence $C\mathcal{L}_s(\bm{u}_{\mathring{\Gamma}_c})-C\mathcal{L}_s(\bm{u}_{\mathring{\Gamma}_c}^*)=\sum_{i\in\mathring{\Gamma}_c}[C\ell_{s}(\bm{u}_i)-C\ell_{s}(\bm{u}^*_i)]\geq 0$. 

In summary, (\ref{sim_inequ}) holds for $0<\gamma C<2(v-\epsilon)^2$. \\
Case II: For $\gamma C\geq 2(v-\epsilon)^2$.
Denote $\bm{s}^*=\bm{u}^*-\gamma\bm{\lambda}^*$ and 
\begin{equation}\label{def_set}
    \begin{aligned}
        &\Xi_1^*:=\{i\ \vert\ \bm{s}^*_i\leq \epsilon\};\\
        &\Xi_2^*:=\{i\ \vert\ \epsilon<\bm{s}^*_i<\sqrt{2\gamma C}+\epsilon\}\cup\{i\ \vert\ \bm{s}^*_i=\sqrt{2\gamma C}+\epsilon, \bm{\lambda}^*_i\neq0\};\\
        &\Xi_3^*:=\{i\ \vert\ \bm{s}^*_i>\sqrt{2\gamma C}+\epsilon\}\cup \{i\ \vert\ \bm{s}^*_i=\sqrt{2\gamma C}+\epsilon, \bm{\lambda}^*_i= 0\}.
    \end{aligned}
\end{equation}
Similar to the previous discussion, we have 
\begin{equation*}
\begin{aligned}
    &\bm{u}^*_{\Xi_1^*}=(\textup{Prox}_{\gamma C\mathcal{L}_s}(\bm{u}^*-\gamma\bm{\lambda}^*))_{\Xi_1^*}=(\bm{u}^*-\gamma\bm{\lambda}^*)_{\Xi_1^*};\\
    &\bm{u}^*_{\Xi_2^*}=(\textup{Prox}_{\gamma C\mathcal{L}_s}(\bm{u}^*-\gamma\bm{\lambda}^*))_{\Xi_2^*}=\epsilon;\\
    &\bm{u}^*_{\Xi_3^*}=(\textup{Prox}_{\gamma C\mathcal{L}_s}(\bm{u}^*-\gamma\bm{\lambda}^*))_{\Xi_3^*}=(\bm{u}^*-\gamma\bm{\lambda}^*)_{\Xi_3^*},
\end{aligned}
\end{equation*}
and hence 
\begin{equation*}
    \begin{cases}
        &\bm{\lambda}^*_{\Xi_1^*}=0;\\
        &\bm{u}^*_{\Xi_2^*}=\epsilon;\\
        &\bm{\lambda}^*_{\Xi_3^*}=0.
    \end{cases}
\end{equation*}
Immediately, following from (\ref{def_set}), we can obtain   
\begin{equation}\label{geq2v}
    \begin{cases}
    &\bm{\lambda}^*_i=0,\quad \bm{u}^*_i\leq \epsilon,\quad i\in \Xi_1^*\\
    &-\sqrt{\frac{2C}{\gamma}}\leq\bm{\lambda}^*_i<0,\quad \bm{u}_i^*=\epsilon,\quad i\in \Xi_2^*\\
    &\bm{\lambda}^*_i=0,\quad \bm{u}^*_i\geq\sqrt{2\gamma C}+\epsilon,\quad i\in\Xi_3^*,
    \end{cases}
\end{equation}
and hence $-\sqrt{\frac{2C}{\gamma}}\leq\bm{\lambda}^*_i\leq 0$ for $\gamma C\geq 2(v-\epsilon)^2$.

Define $\mathring{\Xi}:=\Xi_2^*$ and $\mathring{\Xi}_c:=\Xi_1^*\cup \Xi_3^*$. We will construct that 
\begin{equation*}
\begin{aligned}
   & C\mathcal{L}_s(u_{\mathring{\Xi}})-C\mathcal{L}_s(u_{\mathring{\Xi}}^*)+\langle u_{\mathring{\Xi}}-u_{\mathring{\Xi}}^*,\bm{\lambda}^*_{\mathring{\Xi}}\rangle\geq 0;\\
   & C\mathcal{L}_s(u_{\mathring{\Xi}_c})-C\mathcal{L}_s(u_{\mathring{\Xi}_c}^*)\geq 0.
\end{aligned}    
\end{equation*}
Since $\bm{u}^*_i=\epsilon$ for $i\in\mathring{\Xi}$, we have 
\begin{equation*}
    \bm{u}^*_i-\delta_m\leq u_i\leq \bm{u}^*_i+\delta_m<v
\end{equation*}
for any $u_i$ satisfying $\vert u_i-\bm{u}^*_i\vert\leq \delta_m$, and then $\ell_{s}(u_i)\geq \frac{u_i-\epsilon}{v-\epsilon}$ and $\ell_{s}(u_i^*)=0$. 
If $\epsilon\leq u_i\leq \epsilon+\delta_m$, we get that 
\begin{equation*}
\begin{aligned}
    &C\ell_{s}(u_i)-C\ell_{s}(u_i^*)+\bm{\lambda}^*_i(u_i-\bm{u}^*_i)
    \geq C(\frac{u_i}{v-\epsilon}-\frac{\bm{u}^*_i}{v-\epsilon})+\bm{\lambda}^*_i(u_i-\bm{u}^*_i)
    =(u_i-\bm{u}^*_i)(\frac{C}{v-\epsilon}+\bm{\lambda}^*_i)\geq 0.\\
   & (\frac{C}{v-\epsilon}+\bm{\lambda}^*_i\geq \frac{C}{v-\epsilon}-\sqrt{\frac{2C}{\gamma}}=\sqrt{\frac{C}{\gamma}}(\frac{\sqrt{\gamma C}}{v-\epsilon}-\sqrt{2}\geq 0)
\end{aligned}
\end{equation*}
If $\epsilon-\delta_m\leq u_i<\epsilon$, we get that $\ell_{s}(u_i)=\ell_{s}(\bm{u}^*_i)=0$ and then
\begin{equation*}
\begin{aligned}
    C\ell_{s}(u_i)-C\ell_{s}(u_i^*)+\bm{\lambda}^*_i(u_i-\bm{u}^*_i)
    = 0+\bm{\lambda}^*_i (u_i-\epsilon)\geq 0.
\end{aligned}
\end{equation*}
Hence $C\mathcal{L}_s(u_{\mathring{\Xi}})-C\mathcal{L}_s(u_{\mathring{\Xi}}^*)+\langle u_{\mathring{\Xi}}-u_{\mathring{\Xi}}^*,\bm{\lambda}^*_{\mathring{\Xi}}\rangle=\sum_{i\in\mathring{\Xi}}[C\ell_{s}(u_i)-C\ell_{s}(u_i^*)+\bm{\lambda}^*_i(u_i-\bm{u}^*_i)]\geq 0$.

Since $\bm{u}^*_i\leq \epsilon$ for $i\in\Xi_1^*$, we have 
\begin{equation*}
   u_i\leq \bm{u}^*_i+\delta_m<v
\end{equation*}
for any $u_i$ satisfying $\vert u_i-\bm{u}^*_i\vert\leq \delta_m$, and then $C\ell_{s}(u_i)\geq C\ell_{s}(u_i^*)=0$. Since $\bm{u}^*_i\geq \sqrt{2\gamma C}+\epsilon\geq 2v-\epsilon$ for $i\in\Xi_3^*$, we have 
\begin{equation*}
   u_i\geq \bm{u}^*_i-\delta_m>v
\end{equation*}
for any $u_i$ satisfying $\vert u_i-\bm{u}^*_i\vert\leq \delta_m$, and then $C\ell_{s}(u_i)= C\ell_{s}(u_i^*)=C$. Hence $C\mathcal{L}_s(u_{\mathring{\Xi}_c})-C\mathcal{L}_s(u_{\mathring{\Xi}_c}^*)=\sum_{i\in\mathring{\Xi}_c}[C\ell_{s}(u_i)-C\ell_{s}(u_i^*)]\geq 0$.
In summary, (\ref{sim_inequ}) holds for $\gamma C\geq 2(v-\epsilon)^2$.

By amalgamating  Case I with Case II, it follows that $(\bm{w}^*;b^*;\bm{u}^*)$ is a local minimizer, and hence we complete the proof.
\end{proof}

\section{Fast Algorithm}
In this section, we introduce the concept of support vectors in our $\ell_s$-SVM classifier.  By utilizing them as the selected working set during the updating of all sub-problems, we devise a fast ADMM algorithm to solve problem \eqref{orig_pro}.
 Through the strategic combination of ADMM with carefully selected working sets, we aim to enhance the optimization process and address the challenges posed by the non-convex non-smooth $\ell_s$-SVM model.
\subsection{$\ell_s$ Support Vectors}
Support vectors play a crucial role in SVM. In classification tasks using SVM, the final classifier is mainly influenced by those samples in the training dataset that are closest to the classification hyperplane. These samples participate in determining the decision classification hyperplane and are thus referred to as support vectors. Next, leveraging the concept of proximal stationary point, we offer a clear definition of support vectors in our proposed $\ell_s$-SVM classifier.
\begin{theorem}[$\ell_s$ Support Vectors for $0<\gamma C<2(v-\epsilon)^2$]\label{theo-sv1}
For $0<\gamma C<2(v-\epsilon)^2$, if $(\bm{w}^*,b^*,\bm{u}^*)$ with $\bm{\lambda}^*\in\mathbb{R}^m$ and $\gamma>0$ is a proximal stationary point of  \eqref{orig_pro} , then we obtain 
    \begin{equation}\label{case1}
    \bm{w}^*=-\sum_{i\in T^*}\bm{\lambda}^*_iy_ix_i,\quad \bm{\lambda}^*_i=0\ \textup{for}\ i\in T^*_c
    \end{equation}
where $T^*:=\{i\ \vert\ \bm{\lambda}^*_i\in[-\frac{C}{v-\epsilon},0)\}$. The training vectors $\{x_i\ \vert\ i\in T^*\}$ are called the $\ell_s$ support vectors. For any $i\in T^*$, the $\ell_s$ support vector $x_i$ satisfies 
\begin{equation*}
    \begin{cases}
    y_i(\langle \bm{w}^*, x_i\rangle +b^*)=1-\epsilon,\ i\in T_1^*:=\{i\in T^*\ :\ \bm{\lambda}^*_i\in(-\frac{C}{v-\epsilon},0)\}\\
    y_i(\langle \bm{w}^*, x_i\rangle +b^*)\in[1+\frac{\gamma C}{2(v-\epsilon)}-v,1],\ i\in T^*_2:=\{i\in T^*\ :\ \bm{\lambda}^*_i=-\frac{C}{v-\epsilon}\}.
    \end{cases}
\end{equation*}
\end{theorem}
\begin{proof}
From the derived results (\ref{lam_range}), it is evident that $\bm{\lambda}^*_i\in[-\frac{C}{v-\epsilon},0]$, $i\in\mathbb{N}$, and hence $\bm{\lambda}^*_i=0$ for $i\in T^*_c$. Additionally, based on the relations in (\ref{0_2v}), we establish that  $T^*=T^*_1\cup T^*_2$ with $T^*_1=\Gamma^*_2$ and  $T^*_2=\Gamma^*_3$. Utilizing $\bm{w}^*+A^{\top}\bm{\lambda}^*=0$ and $A=[y_1 x_1,y_2 x_2,\cdots, y_m x_m]^{\top}$, we can express $\bm{w}^*$ as 
\begin{equation*}
    \bm{w}^*=-A^{\top}_{T^*}\bm{\lambda}^*_{T^*}-A^{\top}_{{T}^*_c}\bm{\lambda}^*_{{T}^*_c}=-A^{\top}_{{T}^*}\bm{\lambda}^*_{{T}^*}=-\sum_{i\in T^*}\bm{\lambda}^*_i y_i x_i.
\end{equation*}
Furthermore, given that $\bm{u}_i^*=\epsilon$ for $i\in T_1^*$ and $\bm{u}_i^*\in[\epsilon,v-\frac{\gamma C}{2(v-\epsilon)}]$ for $i\in T_2^*$, and considering $\bm{u}^*+A\bm{w}^*+b^*\bm{y}=1$, we deduce 
\begin{equation*}
\begin{cases}
(A\bm{w}^*+b^*\bm{y})_i=1-\epsilon,\ i\in T^*_1;\\
(A\bm{w}^*+b^*\bm{y})_i\in [1+\frac{\gamma C}{2(v-\epsilon)}-v,1],\ i\in T^*_2.
\end{cases}
\end{equation*}
Thus, we complete the proof.
\end{proof}
\begin{theorem}[$\ell_s$ support vectors for $\gamma C\geq 2(v-\epsilon)^2$]\label{theo-sv2}
For $\gamma C\geq 2(v-\epsilon)^2$, if $(\bm{w}^*,b^*,\bm{u}^*)$ with $\bm{\lambda}^*\in\mathbb{R}^m$ and $\gamma>0$ is a proximal stationary point of  \eqref{orig_pro} , then $\bm{w}^*$ satisfies 
    \begin{equation}\label{case2}
    \bm{w}^*=-\sum_{i\in T^*}\bm{\lambda}^*_iy_ix_i,\quad \bm{\lambda}^*_i=0\ \textup{for}\ i\in T^*_c
    \end{equation}
where $T^*:=\{i\ \vert\ \bm{\lambda}^*_i\in[-\sqrt{\frac{2C}{\gamma}},0)\}$. The training vectors $\{x_i\ \vert\ i\in T^*\}$ are called the $\ell_s$ support vectors. For any $i\in T^*$, the $\ell_s$ support vector $x_i$ satisfies 
\begin{equation*}
    y_i(\langle \bm{w}^*,x_i\rangle+b^*)=1-\epsilon.
\end{equation*}
\end{theorem}
\begin{proof}
The results derived in (\ref{lam_range}) indicate $\lambda_i^*\in[-\sqrt{\frac{2C}{\gamma}},0]$ for $i\in\mathbb{N}$, and hence $\bm{\lambda}^*_i=0$ for $i\in T^*_c$. Basing on (\ref{geq2v}), we establish $T^*=\Xi_2^*$. By incorporating $\bm{w}^*+A^{\top}\bm{\lambda}^*=0$ and $A=[y_1 x_1,y_2 x_2,\cdots, y_m x_m]^{\top}$, we derive 
\begin{equation*}
    \bm{w}^*=-A^{\top} _{T^*}\bm{\lambda}^*_{T^*}=-\sum_{i\in T^*}\bm{\lambda}^*_i y_ix_i.
\end{equation*}
Besides, (\ref{geq2v}) shows $\bm{u}_i=\epsilon$ for $i\in T^*$, which together with $\bm{u}^*+A\bm{w}^*+b^*\bm{y}=\bm{1}$ yield
\begin{equation*}
    (A\bm{w}^*+b^*\bm{y})_i=1-\epsilon
\end{equation*}
for $i\in {T^*}$. Hence we complete the proof.
\end{proof}
\subsection{$\ell_s$-ADMM Framework}
Building upon the theoretical findings from the previous  subsection, we aim to devise an efficient method for the proposed $\ell_s$-SVM classifier model. Motivated by the explicit expression of $\ell_s$-support vectors, we seek to avoid involving all samples in the algorithm's iterations, as this would lead to significant computational complexity, particularly with large-scale training datasets. To address this challenge, we introduce a method where only a subset of samples participates in updating decision variables, and leverage the Alternating Direction Method of Multipliers (ADMM) in conjunction with the technique of working sets to effectively address the $\ell_s$-SVM problem \eqref{orig_pro}. We refer to this approach as the $\ell_s$-ADMM algorithm. 

Given a positive parameter $\delta$, the augmented Lagrangian function of  \eqref{orig_pro}  is 
\begin{equation*}
    L_{\delta}(\bm{w},b,\bm{u},\bm{\lambda})=\frac{1}{2}\Vert \bm{w}\Vert ^2_2+C\mathcal{L}(\bm{u})+\langle\bm{\lambda},\bm{u}+A\bm{w}+b\bm{y}-1\rangle+\frac{\delta}{2}\Vert \bm{u}+A\bm{w}+b\bm{y}-1\Vert_2^2
\end{equation*}
where $\bm{\lambda}$ is Lagrangian multiplier. Fixed the $k$-th iteration points $(\bm{w}^k;b^k;\bm{u}^k;\bm{\lambda}^k)$, we update the $k+1$-th iteration of $\ell_s$-ADMM with following rules:
\begin{equation}\label{subpro}
    \begin{cases}
    \bm{u}^{k+1}=\arg\min L_{\delta}(\bm{w}^k,b^k,\bm{u},\bm{\lambda}^k)\\
    \bm{w}^{k+1}=\arg\min L_{\delta}(\bm{w},b^k,\bm{u}^{k+1},\bm{\lambda}^k)+\frac{\delta}{2}\Vert \bm{w}-\bm{w}^k\Vert^2_{\mathcal{D}^k}\\
    b^{k+1}=\arg\min L_{\delta}(\bm{w}^{k+1},b,\bm{u}^{k+1},\bm{\lambda}^k)
    \end{cases}
\end{equation}
where $\mathcal{D}^k$ represents the symmetric matrix. The selection of $\mathcal{D}^k$ is based on two considerations: (a) To solve $\bm{w}^{k+1}$ exactly, it is necessary to maintain the convexity. (b) The analysis in \Cref{theo-sv1} and \Cref{theo-sv2} indicate a small portion of training set impacts on optimal hyperplane, which drives us to construct the working set in each iteration step to reduce the computational complexity. 

For convenience, some notations are listed. 
Define $\bm{z}^k:=1-A\bm{w}^k-b^k\bm{y}-\frac{\bm{\lambda}^k}{\delta}$; for $0<\gamma C<2(v-\epsilon)^2$, $T_k^1:=\{i:\epsilon<\bm{z}_i^k<\frac{\gamma C}{v-\epsilon}+\epsilon\}$, $T_k^2:=\{i:\frac{\gamma C}{v-\epsilon}+\epsilon\leq \bm{z}_i^k<v+\frac{\gamma C}{2(v-\epsilon)}\}\cup\{i:\bm{z}_i^k=v+\frac{\gamma C}{2(v-\epsilon)}, \bm{\lambda}_i^k\neq 0\}$; for $\gamma C\geq 2(v-\epsilon)^2$, $T^3_k:=\{i:\bm{z}_i^k\in (\epsilon,\sqrt{\frac{2C}{\delta}}+\epsilon)\}\cup\{i:\bm{z}_i^k=\sqrt{\frac{2C}{\delta}}+\epsilon, \bm{\lambda}_i^k\neq 0\}$; 
We design the working set $T_k$ at the $k$-th step as:
\begin{equation}\label{defworkset}
T_k:=
\begin{cases} 
       T_k^1\cup T_k^2,\quad &\textup{for}\ 0<\gamma C<2(v-\epsilon)^2 \\
       T^3_k,\quad &\textup{for}\ \gamma C\geq 2(v-\epsilon)^2
    \end{cases}
\end{equation}
and then $\mathcal{D}^k$ is given by $\mathcal{D}^k=-A_{{T_k}_c}^{\top}A_{{T_k}_c}$.
Moreover, inspired by (\ref{case1}) and (\ref{case2}), the update rule of multiplier $\bm{\lambda}^{k+1}$ is 
    \begin{equation}\label{up_lambda}
        \begin{cases}
&\lambda_{T_k}^{k+1}=\lambda_{T_k}^{k}+\eta\delta(\bm{u}^{k+1}+Aw^{k+1}+b^{k+1}\bm{y}-1)_{T_k}\\
      &\lambda_{{T_k}_c}^{k+1}=0
        \end{cases}
    \end{equation}
where step-size parameter $\eta\in(0,\frac{1+\sqrt{5}}{2})$. In the following, we give the analytic solution for subproblems (\ref{subpro}):
\begin{itemize}
    \item [(i)] Updating $\bm{u}^{k+1}$. The $\bm{u}$-subproblem can be written as 
    \begin{equation*}
        \begin{aligned}
            \bm{u}^{k+1}&=\arg\min_{\bm{u}}\{C\mathcal{L}(\bm{u})+\langle \bm{\lambda}^k, \bm{u}\rangle+\frac{\delta}{2}\Vert \bm{u}+Aw^k+b^ky-1\Vert_2^2\}\\
            &=\arg\min_{\bm{u}}\{C\mathcal{L}(\bm{u})+\frac{\delta}{2}\Vert \bm{u}-z^k\Vert^2_2\}\\
            &=\textup{Prox}_{\frac{C}{\delta}\mathcal{L}}(z^k)
        \end{aligned}
    \end{equation*}
 Then (\ref{def_prox1}) and (\ref{solu_geq}) show that  
\begin{equation}\label{up_u1}
        \begin{cases}
        u_{T_k^1}^{k+1}=\epsilon\\
        u_{T_k^2}^{k+1}=z_{T_k^2}^{k}-\frac{C}{\delta (v-\epsilon)}\\
        u_{{T_k}_c}^{k+1}=z^k_{{T_k}_c},
        \end{cases}
    \end{equation}
 for $0<\frac{C}{\delta}<2 (v-\epsilon)^2$    
 and 
    \begin{equation}\label{up_u}
        \begin{cases}
        u_{T_k}^{k+1}=\epsilon\\  
        u_{{T_k}_c}^{k+1}=z^k_{{T_k}_c}
        \end{cases}
    \end{equation}
    for $\frac{C}{\delta}\geq 2(v-\epsilon)^2$.
    \item[(ii)] Update $\bm{w}^{k+1}$. The $\bm{w}$-subproblem can be written as 
    \begin{equation*}
        \bm{w}^{k+1}=\arg\min_{\bm{w}}\{\frac{1}{2}\Vert \bm{w}\Vert_2^2+\langle\bm{\lambda}^k,Aw\rangle+\frac{\delta}{2}\Vert \bm{u}^{k+1}+Aw+b^ky-1\Vert_2^2+\frac{\delta}{2}\Vert \bm{w}-\bm{w}^k\Vert_{\mathcal{D}^k}^2\}   
    \end{equation*}
    In view of $A^{\top}A=A^{\top}_{T_k}A_{T_k}+A^{\top}_{{T_k}_c}A_{{T_k}_c}$, we have 
    \begin{equation}\label{solve_w}
        (I+\delta A^{\top}_{T_k} A_{T_k})\bm{w}^{k+1}+\delta A^{\top}(\frac{\bm{\lambda}^k}{\delta}+\bm{u}^{k+1}+b^ky-1)+\delta A^{\top}_{{T_k}_c}A_{{{T_k}_c}}\bm{w}^k=0,
    \end{equation}
    which combining with the fact that 
    \begin{equation*}
    \begin{aligned}
    A^{\top}(\frac{\bm{\lambda}^k}{\delta}+\bm{u}^{k+1}+b^ky-\bm{1})=&\sum_{i\in T_k}A^{\top}_i(\frac{\bm{\lambda}^k}{\delta}+\bm{u}^{k+1}+b^ky-1)_i+\sum_{i\in {T_k}_c}A^{\top}_i(\frac{\bm{\lambda}^k}{\delta}+\bm{u}^{k+1}+b^ky-1)_i\\
        =&\sum_{i\in T_k}A^{\top}_i(\frac{\bm{\lambda}^k}{\delta}+\bm{u}^{k+1}+b^ky-1)_i-\sum_{i\in {T_k}_c}A^{\top}_iA_iw^k
    \end{aligned}
    \end{equation*}
    yield 
    \begin{equation}\label{up_w}
      (I+\delta A^{\top}_{T_k} A_{T_k})\bm{w}^{k+1}+\delta A_{T_k}^{\top}(\frac{\bm{\lambda}^k}{\delta}+\bm{u}^{k+1}+b^ky-1)_{T_k}=0.
    \end{equation}
    If $n\leq \vert T_k\vert$, we have 
    \begin{equation}\label{up_w1}
        \bm{w}^{k+1}=-\delta(I+\delta A^{\top}_{T_k} A_{T_k})^{-1}A^{\top}_{T_k}(\frac{\bm{\lambda}^k}{\delta}+\bm{u}^{k+1}+b^ky-1)_{T_k}; \end{equation} 
    and if $n>\vert T_k\vert$, the Sherman-Morrison-Woodbury formula~\cite{golub2013matrix} yields $(I+\delta A^{\top}_{T_k} A_{T_k})^{-1}=I-\delta A^{\top}_{T_k}(I+\delta A_{T_k} A_{T_k}^{\top})^{-1}A_{T_k}$, and hence by directly calculating, we obtain
    \begin{equation}\label{up_w2}
        \begin{aligned}
        \bm{w}^{k+1}
        =&-\delta A_{T_k}^{\top}[I-\delta(I+\delta A_{T_k}A^{\top}_{T_k} )^{-1}A_{T_k}A_{T_k}^{\top}](\frac{\bm{\lambda}^k}{\delta}+\bm{u}^{k+1}+b^k\bm{y}-\bm{1})_{T_k}\\
        =&-\delta A_{T_k}^{\top}(I+
        \delta A_{T_k}A_{T_k}^{\top})^{-1}(\frac{\bm{\lambda}^k}{\delta}+\bm{u}^{k+1}+b^k\bm{y}-\bm{1})_{T_k}.
        \end{aligned}
    \end{equation}
    \item[(iii)] Update $b^{k+1}$. The $b$-subproblem can be written as 
    \begin{equation*}
        b^{k+1}=\arg\min_{b}\{\langle \bm{\lambda}^k, by\rangle+\frac{\delta}{2}\Vert \bm{u}^{k+1}+A\bm{w}^{k+1}+b\bm{y}-\bm{1}\Vert_2^2\}
    \end{equation*}
    We have 
    \begin{equation*}
        \langle \bm{\lambda}^k, \bm{y}\rangle +\delta \bm{y}^{\top}(\bm{u}^{k+1}+A\bm{w}^{k+1}+b^{k+1}\bm{y}-\bm{1})=0,
    \end{equation*}
    and then 
    \begin{equation}\label{up_b}
      b^{k+1}=\frac{\langle \bm{y},1-\bm{u}^{k+1}-Aw^{k+1}-\frac{\lambda_k}{\delta}\rangle}{m}.  
    \end{equation}
\end{itemize}  

We present the specific details of the $\ell_s$-ADMM in Algorithm \ref{alg:admm_slide}. 
\begin{algorithm}
\renewcommand{\algorithmicrequire}{\textbf{Input:}}
\renewcommand{\algorithmicensure}{\textbf{Output:}}
\caption{$\ell_s$-ADMM for solving \eqref{orig_pro}}
\label{alg:admm_slide}
\begin{algorithmic}[1]
\REQUIRE
\quad
\\ Regularized parameters $C, \delta$; Slide loss function parameters $\epsilon, v$;  stepsize parameter $\eta$; maximal iteration $K$.
\ENSURE the decision hyperplane parameter $(\bm{w}^*;b^*)$.
\STATE{\textbf{Initilization:} $(\bm{w}^0;b^0;\bm{u}^0;\bm{\lambda}^0)$;  k=0}
\REPEAT
\STATE{Updating $T_k$ by \eqref{defworkset} ;}
\STATE{Updating $\bm{U}^{k+1}$ by \eqref{up_u1} and \eqref{up_u};}
\STATE{Updating$\bm{w}^{k+1}$ by \eqref{up_w1} and \eqref{up_w2} ;}
\STATE{Updating $b^{k+1}$ by \eqref{up_b};}
\STATE{Updating $\bm{\lambda}^{k+1}$ by \eqref{up_lambda};}
\STATE{k=k+1;}
\UNTIL{The termination criterion is satisfied or $k>K$}
\RETURN $(\bm{w}^*;b^*)=(\bm{w}^k;b^k)$
\end{algorithmic}
\end{algorithm}

\subsection{Convergence Analysis}
Next, we provide the convergence analysis of Algorithm \ref{alg:admm_slide}.
\begin{theorem}
Suppose $(\bm{w}^*,b^*,\bm{u}^*,\bm{\lambda}^*)$ be the limit point of the sequence $\{(\bm{w}^k,b^k,\bm{u}^k,\bm{\lambda}^k)\}$ generated by $\ell_s$-ADMM method. Then $(\bm{w}^*,b^*,\bm{u}^*,\bm{\lambda}^*)$ is a proximal stationary point with $\gamma=\frac{1}{\delta}$ and also a locally optimal solution to the problem  \eqref{orig_pro} . 
\end{theorem}
\begin{proof} 
Firstly, considering the case that the set $\Lambda_1:=\{k\ |\ T_k=\emptyset\}$ is a finite subset of $\mathbb{N}$, i.e., $\vert\Lambda_1\vert<\infty$, we need to further discuss whether the set $\Lambda_2:=\{k\ |\  (T_k)_c=\emptyset,\ k\in\mathbb{N}\setminus \Lambda_1\}$ is a finite subset. 
\begin{itemize}
    \item[(A)] If $\Lambda_2$ is a finite subset, we have $T_k\neq\emptyset$ and $(T_k)_c\neq\emptyset$ for any $k\in\mathbb{N}\setminus(\Lambda_1\cup\Lambda_2)$. Observing that the number of elements of index set $T_k$ is finite for any $k\in\mathbb{N}$, we obtain that there exist infinite subset $J\subseteq \mathbb{N}\setminus(\Lambda_1\cup\Lambda_2)$ and a fixed nonempty set $T$ such that $T_j \equiv T$ for any $j\in J$. Taking the limit along with $J$, i.e., $k\in J$ and $k\rightarrow \infty$, we obtain $z^*=1-Aw^*-b^*\bm{y}-\frac{\bm{\lambda}^*}{\delta}$. Moreover, it follows from (\ref{up_lambda}) that
\begin{equation}
\begin{cases}    \bm{\lambda}^*_T=\bm{\lambda}^*_T+\eta\delta(\bm{u}^*+Aw^*+b^*\bm{y}-1)_{T}\\
    \bm{\lambda}^*_{{T}_c}=0,
    \end{cases}
\end{equation}
which indicates $(\bm{u}^*+Aw^*+b^*\bm{y}-1)_{T}=0$, that is $\bm{z}^*_{T}=\bm{u}^*_T-\frac{1}{\delta}\bm{\lambda}^*_T$.
    \begin{itemize}
        \item [(a)] For $0<\frac{C}{\delta}<2(v-\epsilon)^2$. When the set $\Omega_1:=\{k\ |\ T^1_k=\emptyset,\ k\in J\}$ is a finite set, it yields that $T^1_k\neq\emptyset$ for any $k\in J\setminus \Omega_1$. 
        \begin{itemize}
            \item [(i)] 
        If the set $\Omega_2:=\{k\ |\ T^2_k=\emptyset,\ k\in J\setminus \Omega_1\}$ is a finite set, $T_k^1\neq\emptyset$ and $T^2_k\neq\emptyset$ for any $k
        \in J\setminus(\Omega_1\cup\Omega_2)$. Since $T_k^1$ is a finite set for any $k\in J\setminus(\Omega_1\cup\Omega_2)$ , there exists infinite subset $\hat{J}\subseteq J\setminus(\Omega_1\cup\Omega_2)$ and nonempty sets $T^1, T^2$ such that $T_k^1\equiv T^1$, $T_k^2\equiv T^2$ for any $k\in\hat{J}$ and $T^1
        \cup T^2=T$. Taking the limit along with $\hat{J}$, i.e., $k\in \hat{J}$ and $k\rightarrow\infty$, it follows from (\ref{up_u1}) that 
        \begin{equation*}
            \begin{cases}
                \bm{u}^*_{T^1}=\epsilon\\
                \bm{u}^*_{T^2}=\bm{z}^*_{T^2}-\frac{C}{\delta (v-\epsilon)}\\
                \bm{u}^*_{T_c}=\bm{z}^*_{T_c}
            \end{cases}
        \end{equation*}
        which implies $\bm{z}^*_{{T}_c}=\bm{u}^*_{T_c}-\frac{1}{\delta}\bm{\lambda}_{T_c}$,  hence $\bm{z}^*=\bm{u}^*-\frac{1}{\delta}\bm{\lambda}^*$. By directly calculating, we obtain that $\bm{u}^*\in\textup{Prox}_{\frac{C}{\delta}\mathcal{L}}(\bm{z}^*)$, i.e., $\bm{u}^*\in\textup{Prox}_{\frac{C}{\delta}\mathcal{L}}(\bm{u}^*-\frac{1}{\delta}\bm{\lambda}^*)$. 
        \item [(ii)] If the set $\Omega_2:=\{k\ |\ T^2_k=\emptyset,\ k\in J\setminus \Omega_1\}$ is a infinite set, we obtain that $T_k^1=T_k\equiv T$ for any $k\in\Omega_2$. Taking the limit along with $\Omega_2$, i.e., $k\in \Omega_2$ and $k\rightarrow\infty$, it follows from (\ref{up_u1}) that 
        \begin{equation*}
            \begin{cases}
                \bm{u}^*_{T}=\epsilon\\
                \bm{u}^*_{T_c}=\bm{z}^*_{T_c},
            \end{cases}
        \end{equation*}
        which yields $\bm{z}^*=\bm{u}^*-\frac{1}{\delta}\bm{\lambda}^*$, and further implies $\bm{u}^*\in\textup{Prox}_{\frac{C}{\delta}\mathcal{L}}(\bm{u}^*-\frac{1}{\delta}\bm{\lambda}^*)$.
         \end{itemize}
        
        When the set $\Omega_1:=\{k\ |\ T^1_k=\emptyset,\ k\in J\}$ is a infinite set, it yields that $T_k^2=T_k\equiv T$ for any $k\in\Omega_1$. Taking the limit along with $\Omega_1$, i.e., $k\in \Omega_1$ and $k\rightarrow\infty$, it follows from (\ref{up_u1}) that 
        \begin{equation*}
            \begin{cases}
                \bm{u}^*_{T}=\bm{z}^*_{T}-\frac{C}{\delta (v-\epsilon)}\\
                \bm{u}^*_{T_c}=\bm{z}^*_{T_c},
            \end{cases}
        \end{equation*}
        which yields $\bm{z}^*=\bm{u}^*-\frac{1}{\delta}\bm{\lambda}^*$, and further implies $\bm{u}^*\in\textup{Prox}_{\frac{C}{\delta}\mathcal{L}}(\bm{u}^*-\frac{1}{\delta}\bm{\lambda}^*)$.
        
     \item[(b)] For $\frac{C}{\delta}\geq 2(v-\epsilon)^2$. Taking the limit along with $J$, i.e., $k\in J$ and $k\rightarrow \infty$, it follows from (\ref{up_u}) that 
     \begin{equation*}
            \begin{cases}
                \bm{u}^*_{T}=\epsilon\\
                \bm{u}^*_{T_c}=\bm{z}^*_{T_c},
            \end{cases}
        \end{equation*}
        which yields $\bm{z}^*=\bm{u}^*-\frac{1}{\delta}\bm{\lambda}^*$, and further implies $\bm{u}^*\in\textup{Prox}_{\frac{C}{\delta}\mathcal{L}}(\bm{u}^*-\frac{1}{\delta}\bm{\lambda}^*)$.
    \end{itemize}
   Obviously,  the result  $\bm{z}^*=\bm{u}^*-\frac{1}{\delta}\bm{\lambda}^*$ in above all discussions gives $\bm{u}^*+A\bm{w}^*+b^*\bm{y}-\bm{1}=\bm{0}$. Taking the limit along with $J$ in (\ref{up_w}), we obtain 
   \begin{equation*}
 \begin{aligned}
 (I+\delta A^{\top}_{T} A_T)\bm{w}^{*}=&-\delta A_{T}^{\top}(\frac{\bm{\lambda}^*}{\delta}+\bm{u}^{*}+b^* \bm{y}-\bm{1})_{T}\\
 =&-\delta A_{T}^{\top}(\frac{\bm{\lambda}^*}{\delta}-A\bm{w}^*+A\bm{w}^*+\bm{u}^{*}+b^* \bm{y}-\bm{1})_{T}\\
 =&-\delta A_{T}^{\top}(\frac{\bm{\lambda}^*_{T}}{\delta}-A_{T}\bm{w}^*),
 \end{aligned}   
\end{equation*}
which indicates $\bm{w}^*=-A^{\top}_{T}\lambda_T^*=-A^{\top}\bm{\lambda}^*$. 
    \item[(B)] If $\Lambda_2$ is a infinite set, it brings that $T_k\equiv [m]$ for any $k\in \Lambda_2$. Taking the limit along with $\Lambda_2$, i.e., $k\in \Lambda_2$ and $k\rightarrow\infty$, we obtain $\bm{z}^*=\bm{1}-A\bm{w}^*-b^*\bm{y}-\frac{\bm{\lambda}^*}{\delta}$ and  $\bm{\lambda}^*=\bm{\lambda}^*+\eta\delta(\bm{u}^*+Aw^*+b^*\bm{y}-1)$ driving from (\ref{up_lambda}). Thus, $\bm{u}^*+A\bm{w}^*+b^*\bm{y}-\bm{1}=\bm{0}$ and $\bm{z}^*=\bm{u}^*-\frac{1}{\delta}\bm{\lambda}^*$. Under this case, (\ref{up_w}) can be rewritten as 
   \begin{equation*}
      (I+\delta A^{\top} A)\bm{w}^{k+1}+\delta A^{\top}(\frac{\bm{\lambda}^k}{\delta}+\bm{u}^{k+1}+b^k\bm{y}-\bm{1})=0,
    \end{equation*}
  and taking the limit along with $\Lambda_2$, we obtain 
  \begin{equation*}
      \bm{w}^*+\delta A^{\top}(\frac{\bm{\lambda}^*}{\delta}+\bm{u}^*+b^*\bm{y}-\bm{1}+A\bm{w}^*)=\bm{0},
  \end{equation*}
  which implies $\bm{w}^*+A^{\top}\bm{\lambda}^*=\bm{0}$.
    \begin{itemize}
       \item [(a)] For $0<\frac{C}{\delta}<2(v-\epsilon)^2$. When the set $\Omega_1:=\{k\ |\ T^1_k=\emptyset,\ k\in \Lambda_2\}$ is a finite set, it yields that $T^1_k\neq\emptyset$ for any $k\in \Lambda_2\setminus \Omega_1$. 
        \begin{itemize}
            \item [(i)] 
        If the set $\Omega_2:=\{k\ |\ T^2_k=\emptyset,\ k\in \Lambda_2\setminus \Omega_1\}$ is a finite set, $T_k^1\neq\emptyset$   and $T^2_k\neq\emptyset$ for any $k
        \in \Lambda_2\setminus(\Omega_1\cup\Omega_2)$. Since $T_k^1$ is a finite set for any $k\in \Lambda_2\setminus(\Omega_1\cup\Omega_2)$ , there exists infinite subset $\hat{\Lambda}_2\subseteq \Lambda_2\setminus(\Omega_1\cup\Omega_2) $ and nonempty sets $T^1$, $T^2$ such that $T_k^1\equiv T^1$, $T_k^2\equiv T^2$ for any $k\in\hat{\Lambda}_2$  and $T^1
        \cup T^2=[m]$. Taking the limit along with $\hat{\Lambda}_2$, i.e., $k\in \hat{\Lambda}_2$ and $k\rightarrow\infty$, it follows from (\ref{up_u1}) that 
        \begin{equation*}
            \begin{cases}
                \bm{u}^*_{T^1}=\epsilon\\
                \bm{u}^*_{T^2}=\bm{z}^*_{T^2}-\frac{C}{\delta (v-\epsilon)}.
            \end{cases}
        \end{equation*}
 By directly calculating, we obtain that $\bm{u}^*\in\textup{Prox}_{\frac{C}{\delta}\mathcal{L}}(\bm{z}^*)$, i.e., $\bm{u}^*\in\textup{Prox}_{\frac{C}{\delta}\mathcal{L}}(\bm{u}^*-\frac{1}{\delta}\bm{\lambda}^*)$. 
        \item [(ii)] If the set $\Omega_2:=\{k\ |\ T^2_k=\emptyset,\ k\in \Lambda_2\setminus \Omega_1\}$ is a infinite set, we obtain that $T_k^1=T_k\equiv [m]$ for any $k\in\Omega_2$. Taking the limit along with $\Omega_2$, i.e., $k\in \Omega_2$ and $k\rightarrow\infty$, it follows from (\ref{up_u1}) that $\bm{u}^*=\bm{\epsilon}$ which implies $\bm{u}^*\in\textup{Prox}_{\frac{C}{\delta}\mathcal{L}}(\bm{u}^*-\frac{1}{\delta}\bm{\lambda}^*)$.
         \end{itemize}
 
        When the set $\Omega_1:=\{k\ |\ T^1_k=\emptyset,\ k\in \Lambda_2\}$ is a infinite set, it yields that $T_k^2=T_k\equiv [m]$ for any $k\in\Omega_1$. Taking the limit along with $\Omega_1$, i.e., $k\in \Omega_1$ and $k\rightarrow\infty$, it follows from (\ref{up_u1}) that $\bm{u}^*=\bm{z}^*-\frac{C}{\delta (v-\epsilon)}$ which implies $\bm{u}^*\in\textup{Prox}_{\frac{C}{\delta}\mathcal{L}}(\bm{u}^*-\frac{1}{\delta}\bm{\lambda}^*)$.
        
     \item[(b)] For $\frac{C}{\delta}\geq 2(v-\epsilon)^2$. Taking the limit along with $\Lambda_2$, i.e., $k\in \Lambda_2$ and $k\rightarrow \infty$, it follows from (\ref{up_u}) that $\bm{u}^*=\bm{\epsilon}
     $, which implies $\bm{u}^*\in\textup{Prox}_{\frac{C}{\delta}\mathcal{L}}(\bm{u}^*-\frac{1}{\delta}\bm{\lambda}^*)$.
    \end{itemize}
\end{itemize}

Secondly, we consider the case that the set $\Lambda_1:=\{k\ |\ T_k=\emptyset\}$ is a infinite subset of $\mathbb{N}$, i.e., $\vert\Lambda_1\vert=\infty$, which implies that $(T_k)_c=[m]$ for any $k\in \Lambda_1$. Taking the limit along with $\Lambda_1$, i.e., $k\in\Lambda_1$ and $k\rightarrow\infty$, we obtain that $\bm{z}^*=\bm{1}-A\bm{w}^*-b^*\bm{y}-\frac{\bm{\lambda}^*}{\delta}$ and $\bm{\lambda}^*=\bm{0}$ driving from (\ref{up_lambda}). Moreover, it follows from (\ref{up_u1}) and (\ref{up_u}) that $\bm{z}^*=\bm{u}^*$, which further yields $\bm{u}^*+A\bm{w}^*+b^*\bm{y}-\bm{1}=\bm{0}$. By directly calculating, we obtain that $\bm{u}^*\in\textup{Prox}_{\frac{C}{\delta}\mathcal{L}}(\bm{u}^*-\frac{1}{\delta}\bm{\lambda}^*)$. Under this case, (\ref{solve_w}) can be rewritten as 
\begin{equation*}
       \bm{w}^{k+1}+\delta A^{\top}(\frac{\bm{\lambda}^k}{\delta}+\bm{u}^{k+1}+b^k\bm{y}-\bm{1}+A\bm{w}^k)=0, 
    \end{equation*}
then taking the limit along with $\Lambda_1$, we obtain 
\begin{equation*}
    \bm{w}^*+\delta A^{\top}(\frac{\bm{\lambda}^*}{\delta}+\bm{u}^*+b^*\bm{y}-\bm{1}+A\bm{w}^*)=\bm{0},
\end{equation*}
which yields $\bm{w}^*+A^{\top}\bm{\lambda}^*=0$.

Finally, taking the limit along with $k$ in (\ref{up_b}), we obtain 
\begin{equation*}
\begin{aligned}
    b^*=&\frac{\langle \bm{y}, \bm{1}-A\bm{w}^*-\bm{u}^*-\frac{\bm{\lambda}^*}{\delta}\rangle}{m}\\
    =&\frac{\langle \bm{y}, \bm{1}-A\bm{w}^*-\bm{u}^*-b^*\bm{y}+b^*\bm{y}-\frac{\bm{\lambda}^*}{\delta}\rangle}{m}\\
    =&\frac{\langle \bm{y}, b^*\bm{y}-\frac{\bm{\lambda}^*}{\delta}\rangle}{m}\\
    =&b^*-\frac{\langle \bm{y},\frac{\bm{\lambda}^*}{\delta}\rangle}{m},
\end{aligned}
\end{equation*}
which implies $\langle \bm{y}, \bm{\lambda}^*\rangle=0$. 

Basing on above all discussion, we obtain that 
\begin{equation*}
    \begin{cases}
        \bm{w}^*+A^{\top}\bm{\lambda}^*=\bm{0}\\
       \langle \bm{y},\bm{\lambda}^*\rangle=0\\
       \bm{u}^*+A\bm{w}^*+b^*\bm{y}-\bm{1}=\bm{0}\\
       \bm{u}^*\in\textup{Prox}_{\frac{C}{\delta}\mathcal{L}}(\bm{u}^*-\frac{\bm{\lambda}^*}{\delta}).
    \end{cases}
\end{equation*}
Therefore, $(\bm{w}^*,b^*,\bm{u}^*,\bm{\lambda}^*)$ is a proximal stationary point with $\gamma=\frac{1}{\delta}$, and according to \Cref{suf_nec_con}, it is a local minimizer of the problem  \eqref{orig_pro} . This completes the proof.
\end{proof}

\section{Numerical Experiments}
In this section, we conducted numerical experiments on open-source datasets\footnote{Data sources:  https://www.csie.ntu.edu.tw/~cjlin/libsvmtools/datasets/, https://archive.ics.uci.edu/ml/index.php} to demonstrate the effectiveness and robustness of the proposed $\ell_s$-SVM classifier. These datasets include leukemia, vote, splice, adult, cod-rna, phishing, ijcnn1. \Cref{data_svm} summarizes the detailed information about these seven datasets used in the experiments.
\begin{table}[htbp]
\caption{Detailed information of the datasets
}
\vspace{2mm}
\begin{center}
\setlength{\tabcolsep}{1mm}{
% \addvbuffer[-15pt -12pt]{
\label{data_svm}
\begin{tabular}{|c|c|c|c|}
\hline
Dataset &\# Number of Training Samples & \# Feature &\# Number of Test Samples \\
\hline
\hline
leukemia&38&7129&34\\
vote&435&16&0\\
splice&1000&60&2175\\
phishing&11055&68&0\\
adult&32561&123&16281\\
ijcnn1&49990&22&91701\\
cod-rna&59535&8&271617\\
\hline
\end{tabular}
% \tabnote{$^{\rm a}$This footnote shows what footnote symbols to use.}
}
%\vspace{-8mm}
\end{center}
\end{table}

\textbf{Methods to compare.}
We compared the $\ell_s$-SVM classifier with other popular support vector machine (SVM) classifier methods currently available, as detailed in \Cref{svm_solver}. Particularly, to illustrate the necessity of applying different penalty levels to samples lying within the margins of the two-class hyperplane, we considered the RSVM classifier and the $\ell_{s_o}$-SVM classifier. The RSVM is a support vector machine classifier established based on setting parameters $\epsilon=0$ and $v=1$ in the $\ell_s$-SVM framework, while the $\ell_{s_o}$-SVM classifier is established by setting parameters $\epsilon=0$ and $v<1$ in the $\ell_s$-SVM framework.
\begin{table}[htbp]
\caption{Description of compared methods}\label{svm_solver}
\centering
{
\begin{tabular}{c|c}
\toprule 
Solver &  Model\\
\midrule 
$0/1$ SVM & Hard Margin Loss SVM~\cite{wang2021support}\tablefootnote{Code Resource for $0/1$SVM:https://github.com/Huajun-Wang/L01ADMM}\\
SLTSVM & Symmetric LINEX Loss SVM~\cite{si2023symmetric} \tablefootnote{Code Resource for SLTSVM :https://github.com/sqsiqi/SLTSVM}\\
TpinSVM & Truncated Pinball Loss SVM~\cite{singla2021pin}\tablefootnote{Code Resource for TpinSVM :https://github.com/manisha1427/TruncpinTSVM}\\
TLSSVM & Truncated Least Square SVM~\cite{zhou2023unified}\tablefootnote{Code Resource for TLSSVM :https://github.com/stayones/code-UNiSVM/tree/master}\\
RSVM & Ramp Loss SVM \\
$\ell_{s_o}$-SVM& Slide Loss SVM with $\epsilon=0$\\
\bottomrule 
\end{tabular}
}
%\vspace{-2mm}
\end{table}

\textbf{Evaluation criteria.}
To evaluate the performance of all classifiers, we compute the accuracy by calculating the ratio of misclassified samples in the test dataset to the total number of samples. The expression for accuracy is given by:
\begin{equation*}
\textup{Accuracy (acc)}:=1-\frac{1}{2m_{test}}\sum_{i=1}^{m_{test}}\vert\textup{sign}(\langle\bm{w}^*,\bm{x}\rangle+b^*)-y_i\vert,
\end{equation*}
where $m_{test}$ is the total number of test dataset, $\bm{w}^*$ and $b^*$ are the parameters of the decision classification hyperplane obtained, and $\textup{sign}$ denotes the sign function, such that $\textup{sign}(t) = 1$ when $t > 0$; otherwise, $\textup{sign}(t) = 0$. Additionally, we include CPU time as a performance metric.

\textbf{Stopping criteria.}
Motivated by the \Cref{suf_nec_con}, we utilize the proximal stationary point as the termination criterion. The iteration stops immediately when the iteration sequence $(\bm{w}^k;b^k;\bm{u}^k)$ generated by Algorithm \ref{alg:admm_slide} satisfies the following condition:
$$\max\{e_1^k,e_2^k,e_3^k,e_4^k\}<tol$$
where $tol=1e-3$,
$$e_1^k:=\frac{\Vert\bm{w}^k+A^{\intercal}_{T_k}\bm{\lambda}_{T_k}\Vert}{1+\Vert\bm{w}^k\Vert},\quad\quad e_2^k:=\frac{\vert\langle\bm{y}_{T_k},\bm{\lambda}^k_{T_k}\rangle\vert}{1+\vert T_k\vert}$$
$$e_3^k:=\frac{\Vert\bm{1}-\bm{u}^k-A\bm{w}^k-b^k\bm{y}\Vert}{\sqrt{m}},\quad\quad e_4^k:=\frac{\Vert\bm{u}^k-\textup{Prox}_{ C/\delta\mathcal{L}_s}(\bm{u}^k-\bm{\lambda}^k/\delta)\Vert}{1+\Vert\bm{u}^k\Vert}.$$
The termination conditions for the remaining comparison methods are set following the original papers.

\textbf{Parameters setting.} 
In the $\ell_s$-SVM classifier, the regularization parameters $C$ and $\delta$ are selected from the set $\Omega := \{a^{-7}, a^{-6}, \ldots, a^6, a^7\}$, where $a=\sqrt{2}$. The Slide loss function parameter $v$ is chosen from the set $\{0.1, 0.2, \ldots, 0.9, 1\}$, with $\epsilon=v/10$. The step size parameter $\eta=1.618$. The maximum number of iterations $K=1000$. Since the parameters sets of $\lambda$ and $\gamma$ for the TLSSVM classifier are not specified in original paper, we select them from the set $\Omega$ when reproducing the code. The ranges for all parameters of the other comparison methods follow the settings in the original papers. To ensure a fair comparison among different classifier methods, we employ a grid search strategy combined with ten-fold cross-validation to obtain parameters that yield the highest cross-validation accuracy.

\textbf{Experimental result.}
In the following, we apply the classifier methods listed in \Cref{svm_solver} to conduct performance testing on datasets. First, we normalize all sample points to the interval $[-1, 1]$. For datasets without predefined test sets, we conduct ten-fold cross-validation, i.e., using $90\%$ of the samples for training and $10\%$ for testing. We repeat this process ten times and report the average accuracy results. The performance results of all classifiers are shown in \Cref{accresult_svm} and \Cref{cpu_svm}. For the dataset vote and phishing, the CPU time corresponds to the average time taken for one ten-fold cross-validation. ``**" indicates that no result is obtained for the TpinSVM classifier due to its high memory requirements or the iterative runtime exceeding three hours. Since the original papers provide the classification accuracy results for the TLSSVM classifier on the dataset adult and for the $0/1$ SVM on the dataset ijcnn1, we directly cite them here.
\begin{table}[thbp]
\caption{Results of classification accuracy ($\%$) for all support vector machine classifiers.}\label{accresult_svm}
\centering
{
\resizebox{0.9\textwidth}{!}{
\begin{tabular}{c|ccccccc}
\toprule %[2pt]璁剧疆绾垮
Dataset & $\ell_{s}$-SVM &  RSVM & $\ell_{s_o}$-SVM & $0/1$ SVM& SLTSVM & TpinSVM & TLSSVM\\ %鎹㈣
\midrule %[2pt]
leukemia & \bf{91.18}& \bf{91.18}&\bf{91.18}&\bf{91.18}&55.88&52.94&79.41\\
vote& \bf{94.58}&94.53&94.55&85.39&94.53&94.18&93.35\\
splice &84.51&83.82&\bf{85.10}&84.09&84.05&85.06&84.92\\
phishing & \bf{93.98} & 93.97 & 93.43 & 81.10&92.97 &** & 90.99 \\
adult & \bf{84.97} & 84.57 & 84.92 & 84.79&80.87 &**& 83.32 \\
ijcnn1 & \bf{94.70} & 94.60 & 94.57 & 94.33 &90.50& ** & 90.72\\
cod-rna & \bf{93.08} & 93.06 & 93.07 & 93.07 &91.61 &** & 89.25\\
\midrule
Mean & \bf{91.00} & 90.81 & 90.97 & 87.70&84.34 &77.39 & 87.42\\
\bottomrule %[2pt]
\end{tabular}
}}
%\vspace{-2mm}
\end{table}

\begin{table}[thbp]
\caption{Results of CPU Time (seconds) for all support vector machine classifiers.}\label{cpu_svm}
\centering
{
\resizebox{0.9\textwidth}{!}{
\begin{tabular}{c|ccccccc}
\toprule %[2pt]璁剧疆绾垮
Dataset & $\ell_{s}$-SVM &  RSVM & $\ell_{s_o}$-SVM & $0/1$ SVM& SLTSVM & TpinSVM & TLSSVM\\ %鎹㈣
\midrule %[2pt]
leukemia & \bf{0.791} & 0.911 & 0.812 & 0.816 &11.720&54.164 & 0.811 \\
vote & 1.615 & 1.769 & 1.794 & \bf{0.069}&0.159 &6.552 & 0.540 \\
splice & 0.316 & 0.303 & 0.169 & 0.364 &\bf{0.007}&60.344 & 0.192 \\
phishing & \bf{1.673} & 1.906 & 1.809 & 1.728&9.090 &** & 21.310 \\
adult & 7.367 & 5.381 & 6.578 & 13.535&17.869 &**& \bf{0.381} \\
ijcnn1 & \bf{5.374} & 6.291 & 6.033 & 5.496 &10.477& ** & 8.551\\
cod-rna & 2.096 & 3.329 & 3.116 & 2.859 &1.633 &** & \bf{0.077}\\
\bottomrule %[2pt]
\end{tabular}
}}
%\vspace{-2mm}
\end{table}

From the perspective of classification performance, it is evident that the SLTSVM and TpinSVM classifiers among the comparison methods are not suitable for datasets with a small number of training samples and high-dimensional features. Moreover, the TpinSVM classifier is restricted in its training capacity and is unsuitable for tackling classification problems with large-scale samples. Instead, it is better suited for training on small-sized datasets with low-dimensional features. As for the remaining classifiers, including $0/1$ SVM, SLTSVM, TLSSVM, RSVM, and $\ell_{s_o}$-SVM, although they can be trained on large-scale datasets, it is apparent that our proposed $\ell_s$-SVM classifier generally outperforms them.
In particular, the RSVM and $\ell_{s_o}$-SVM classifiers, corresponding to $\epsilon=0$, exhibit inferior performance compared to $\ell_s$-SVM when dealing with large-sample training sets. This discrepancy arises because these methods excessively penalize samples that are close to the classification hyperplane $f(\bm{x})=\pm 1$ and are correctly classified, leading to poor generalization ability and consequently impacting their performance on test sets. Furthermore, while some classifiers in the comparison methods show the lower CPU times on certain datasets, their corresponding classification performance is notably inferior to that of our proposed $\ell_s$-SVM classifier. Therefore, based on the comprehensive analysis, it's clear that the $\ell_s$-SVM classifier has significant advantages over the other methods.

To evaluate the impact of outliers present in real data on different solvers, we flip the labels of the training sets from above datasets with predefined test sets. We set the flipping rates to $r=\{5\%, 15\%\}$. Subsequently, we trained the aforementioned solvers using the flipped data and examined the classification accuracy on the test sets. The final results are recorded in \Cref{flip5_svm} and \Cref{flip15_svm}. 
The results indicate that as $r$ increases, the classification accuracy of all SVM classifiers decreases on most datasets. However, it can be observed that the classification results of the $\ell_s$-SVM classifier remain relatively stable before and after flipping, and its performance surpasses that of all comparison methods. Therefore, the $\ell_s$-SVM classifier is more robust to outliers compared to other classifiers.

\begin{table}[thbp]
\caption{The classification accuracy ($\%$) results for all support vector machine classifiers with a flipping rate of $r=5\%$.}\label{flip5_svm}
\centering
{
\resizebox{0.9\textwidth}{!}{
\begin{tabular}{c|ccccccc}
\toprule %[2pt]璁剧疆绾垮
Dataset & $\ell_{s}$-SVM &  RSVM & $\ell_{s_o}$-SVM & $0/1$ SVM& SLTSVM & TpinSVM & TLSSVM\\ %鎹㈣
\midrule %[2pt]
leukemia & \bf{91.18} & \bf{91.18} & \bf{91.18} & \bf{91.18} &58.82&52.94 & 79.41 \\ 
splice & 84.14 & 84.05 & 84.32 & 84.09 &84.69&84.69 & \bf{85.15} \\

adult & \bf{84.77} & 83.96 & 84.57 & 84.55&81.60 &**& 82.54 \\
ijcnn1 & \bf{93.96} & 93.58 & 93.13 & 93.82 &91.11& ** & 90.76\\
cod-rna & \bf{93.07} & 93.04 & 93.06 & 93.02 &92.52 &** & 91.52\\
\midrule
Mean & \bf{89.42} & 89.16 & 89.25 & 89.33&81.74&68.79& 85.87\\
\bottomrule %[2pt]
\end{tabular}
}}
%\vspace{-2mm}
\end{table}

\begin{table}[thbp]
\caption{The classification accuracy ($\%$) results for all support vector machine classifiers with a flipping rate of $r=15\%$.}\label{flip15_svm}
\centering
{
\resizebox{0.9\textwidth}{!}{
\begin{tabular}{c|ccccccc}
\toprule %[2pt]璁剧疆绾垮
Dataset & $\ell_{s}$-SVM &  RSVM & $\ell_{s_o}$-SVM & $0/1$ SVM& SLTSVM & TpinSVM & TLSSVM\\ %鎹㈣
\midrule %[2pt]
leukemia & \bf{91.18} & 88.24 & 88.24 & 82.35 &52.94&52.94 & 76.47 \\
splice & \bf{83.31} & 82.81 & 83.03 & 82.71 &82.58&82.57 & 82.76 \\
adult &\bf{84.74} & 83.11 & 84.08 & 84.69&83.22 &**& 82.51 \\
ijcnn1 & \bf{92.79} & 48.22 & 48.51 & 92.39 &80.81& ** & 90.83\\
cod-rna & 92.84 & 92.42 & 92.64 & 92.60 &82.02 &** & \bf{93.07}\\
\midrule
Mean & \bf{88.97} & 78.96 & 79.30 & 86.94&76.31 &67.75 &85.12 \\
\bottomrule %[2pt]
\end{tabular}
}}
%\vspace{-2mm}
\end{table}
\section{Conclusion}
In this paper, we address the limitations of existing partial loss functions when applied to support vector machine (SVM) classifiers by introducing a new Slide loss function based on confidence margins.  Leveraging the theory of nonsmooth analysis, we derive the expressions of subdifferential and proximal operator for the Slide loss function and establish the Slide loss support vector machine (SVM) classifier model ($\ell_s$-SVM). With these explicit expressions, we define the proximal stationary points of this model and provide theoretical analysis of optimality conditions. Furthermore, we investigate the support vectors of $\ell_s$-SVM using proximal stationary points,  laying the foundation for subsequent algorithmic research. We develop an $\ell_s$-ADMM algorithm with a working set based on these support vectors and conduct relevant convergence analysis.  Finally, the robustness and effectiveness of the $\ell_s$-SVM classifier are validated through numerical experiments.

% or


%\section*{Acknowledgment}
%***

\ifCLASSOPTIONcaptionsoff
  \newpage
\fi



%\bibliographystyle{IEEEtran}
\bibliographystyle{IEEEtran}
\bibliography{main}

\begin{thebibliography}{10}
\providecommand{\url}[1]{#1}
\csname url@samestyle\endcsname
\providecommand{\newblock}{\relax}
\providecommand{\bibinfo}[2]{#2}
\providecommand{\BIBentrySTDinterwordspacing}{\spaceskip=0pt\relax}
\providecommand{\BIBentryALTinterwordstretchfactor}{4}
\providecommand{\BIBentryALTinterwordspacing}{\spaceskip=\fontdimen2\font plus
\BIBentryALTinterwordstretchfactor\fontdimen3\font minus \fontdimen4\font\relax}
\providecommand{\BIBforeignlanguage}[2]{{%
\expandafter\ifx\csname l@#1\endcsname\relax
\typeout{** WARNING: IEEEtran.bst: No hyphenation pattern has been}%
\typeout{** loaded for the language `#1'. Using the pattern for}%
\typeout{** the default language instead.}%
\else
\language=\csname l@#1\endcsname
\fi
#2}}
\providecommand{\BIBdecl}{\relax}
\BIBdecl

\bibitem{cortes1995support}
C.~Cortes and V.~Vapnik, ``Support-vector networks,'' \emph{Machine learning}, vol.~20, pp. 273--297, 1995.

\bibitem{cervantes2020comprehensive}
J.~Cervantes, F.~Garcia-Lamont, L.~Rodr{\'\i}guez-Mazahua, and A.~Lopez, ``A comprehensive survey on support vector machine classification: Applications, challenges and trends,'' \emph{Neurocomputing}, vol. 408, pp. 189--215, 2020.

\bibitem{brooks2011support}
J.~P. Brooks, ``Support vector machines with the ramp loss and the hard margin loss,'' \emph{Operations research}, vol.~59, no.~2, pp. 467--479, 2011.

\bibitem{DBLP:books/daglib/0097035}
V.~Vapnik, \emph{Statistical learning theory}.\hskip 1em plus 0.5em minus 0.4em\relax Wiley, 1998.

\bibitem{evgeniou2000regularization}
T.~Evgeniou, M.~Pontil, and T.~Poggio, ``Regularization networks and support vector machines,'' \emph{Advances in computational mathematics}, vol.~13, pp. 1--50, 2000.

\bibitem{wang2020comprehensive}
Q.~Wang, Y.~Ma, K.~Zhao, and Y.~Tian, ``A comprehensive survey of loss functions in machine learning,'' \emph{Annals of Data Science}, pp. 1--26, 2020.

\bibitem{yin2014fault}
S.~Yin, X.~Zhu, and C.~Jing, ``Fault detection based on a robust one class support vector machine,'' \emph{Neurocomputing}, vol. 145, pp. 263--268, 2014.

\bibitem{zhang2001text}
T.~Zhang and F.~J. Oles, ``Text categorization based on regularized linear classification methods,'' \emph{Information retrieval}, vol.~4, pp. 5--31, 2001.

\bibitem{wang2008hybrid}
L.~Wang, J.~Zhu, and H.~Zou, ``Hybrid huberized support vector machines for microarray classification and gene selection,'' \emph{Bioinformatics}, vol.~24, no.~3, pp. 412--419, 2008.

\bibitem{jumutc2013fixed}
V.~Jumutc, X.~Huang, and J.~A. Suykens, ``Fixed-size pegasos for hinge and pinball loss svm,'' in \emph{The 2013 International Joint Conference on Neural Networks (IJCNN)}.\hskip 1em plus 0.5em minus 0.4em\relax IEEE, 2013, pp. 1--7.

\bibitem{liang2021support}
Z.~Liang and L.~Zhang, ``Support vector machines with the $\varepsilon$-insensitive pinball loss function for uncertain data classification,'' \emph{Neurocomputing}, vol. 457, pp. 117--127, 2021.

\bibitem{yan2020efficient}
Y.~Yan and Q.~Li, ``An efficient augmented lagrangian method for support vector machine,'' \emph{Optimization Methods and Software}, vol.~35, no.~4, pp. 855--883, 2020.

\bibitem{huang2016solution}
X.~Huang, L.~Shi, and J.~A. Suykens, ``Solution path for pin-svm classifiers with positive and negative $\tau$ values,'' \emph{IEEE transactions on neural networks and learning systems}, vol.~28, no.~7, pp. 1584--1593, 2016.

\bibitem{allen2018katyusha}
Z.~Allen-Zhu, ``Katyusha: The first direct acceleration of stochastic gradient methods,'' \emph{Journal of Machine Learning Research}, vol.~18, no. 221, pp. 1--51, 2018.

\bibitem{zhu2020support}
W.~Zhu, Y.~Song, and Y.~Xiao, ``Support vector machine classifier with huberized pinball loss,'' \emph{Engineering Applications of Artificial Intelligence}, vol.~91, p. 103635, 2020.

\bibitem{hsieh2008dual}
C.-J. Hsieh, K.-W. Chang, C.-J. Lin, S.~S. Keerthi, and S.~Sundararajan, ``A dual coordinate descent method for large-scale linear svm,'' in \emph{Proceedings of the 25th international conference on Machine learning}, 2008, pp. 408--415.

\bibitem{huang2015sequential}
X.~Huang, L.~Shi, and J.~A. Suykens, ``Sequential minimal optimization for svm with pinball loss,'' \emph{Neurocomputing}, vol. 149, pp. 1596--1603, 2015.

\bibitem{wang2022safe}
H.~Wang and Y.~Xu, ``A safe double screening strategy for elastic net support vector machine,'' \emph{Information Sciences}, vol. 582, pp. 382--397, 2022.

\bibitem{wu2007robust}
Y.~Wu and Y.~Liu, ``Robust truncated hinge loss support vector machines,'' \emph{Journal of the American Statistical Association}, vol. 102, no. 479, pp. 974--983, 2007.

\bibitem{xu2017robust}
G.~Xu, Z.~Cao, B.-G. Hu, and J.~C. Principe, ``Robust support vector machines based on the rescaled hinge loss function,'' \emph{Pattern Recognition}, vol.~63, pp. 139--148, 2017.

\bibitem{singla2020robust}
M.~Singla, D.~Ghosh, K.~Shukla, and W.~Pedrycz, ``Robust twin support vector regression based on rescaled hinge loss,'' \emph{Pattern Recognition}, vol. 105, p. 107395, 2020.

\bibitem{shen2017support}
X.~Shen, L.~Niu, Z.~Qi, and Y.~Tian, ``Support vector machine classifier with truncated pinball loss,'' \emph{Pattern Recognition}, vol.~68, pp. 199--210, 2017.

\bibitem{chen2018sparse}
L.~Chen and S.~Zhou, ``Sparse algorithm for robust lssvm in primal space,'' \emph{Neurocomputing}, vol. 275, pp. 2880--2891, 2018.

\bibitem{park2011robust}
S.~Y. Park and Y.~Liu, ``Robust penalized logistic regression with truncated loss functions,'' \emph{Canadian Journal of Statistics}, vol.~39, no.~2, pp. 300--323, 2011.

\bibitem{wang2021support}
H.~Wang, Y.~Shao, S.~Zhou, C.~Zhang, and N.~Xiu, ``Support vector machine classifier via $l_{0/1}$ soft-margin loss,'' \emph{IEEE transactions on pattern analysis and machine intelligence}, vol.~44, no.~10, pp. 7253--7265, 2021.

\bibitem{mohri2018foundations}
M.~Mohri, A.~Rostamizadeh, and A.~Talwalkar, \emph{Foundations of machine learning}.\hskip 1em plus 0.5em minus 0.4em\relax MIT press, 2018.

\bibitem{rockafellar2009variational}
R.~T. Rockafellar and R.~J.-B. Wets, \emph{Variational analysis}.\hskip 1em plus 0.5em minus 0.4em\relax Springer Science \& Business Media, 2009, vol. 317.

\bibitem{golub2013matrix}
G.~H. Golub and C.~F. Van~Loan, \emph{Matrix computations}.\hskip 1em plus 0.5em minus 0.4em\relax JHU press, 2013.

\bibitem{si2023symmetric}
Q.~Si, Z.~Yang, and J.~Ye, ``Symmetric linex loss twin support vector machine for robust classification and its fast iterative algorithm,'' \emph{Neural Networks}, vol. 168, pp. 143--160, 2023.

\bibitem{singla2021pin}
M.~Singla, D.~Ghosh, and K.~Shukla, ``pin-tsvm: A robust transductive support vector machine and its application to the detection of covid-19 infected patients,'' \emph{Neural Processing Letters}, vol.~53, no.~6, pp. 3981--4010, 2021.

\bibitem{zhou2023unified}
S.~Zhou and W.~Zhou, ``Unified svm algorithm based on ls-dc loss,'' \emph{Machine Learning}, vol. 112, no.~8, pp. 2975--3002, 2023.

\end{thebibliography}

% \end{IEEEbiography}


% % insert where needed to balance the two columns on the last page with
% %\newpage



%\vfill


\end{document}